\documentclass[10pt]{article} 
\usepackage{url}
\usepackage{smile}
\usepackage{graphicx} 
\usepackage{algorithm}
\usepackage{algorithmic}
\usepackage{todonotes}
\usepackage{epstopdf}
\usepackage{wrapfig}
\usepackage[colorlinks, linkcolor=blue, anchorcolor=blue, citecolor=blue,hypertexnames=false]{hyperref}
\usepackage[margin=1in]{geometry}
\usepackage[normalem]{ulem}
\usepackage[export]{adjustbox}
\usepackage{mathtools, cuted}
\usepackage{mathrsfs}
\usepackage{natbib}
\usepackage{enumerate}
\usepackage{multirow}
\usepackage{lineno}

\newcommand{\scrE}{\mathscr{E}}
\newcommand{\scrD}{\mathscr{D}}
\newcommand{\scrG}{\mathscr{G}}

\newcommand{\supp}{\mathrm{supp}}

\providecommand{\norm}[1]{\|#1\|}

\newcommand{\ReLU}{\mathrm{ReLU}}

\newtheorem*{theorem*}{Theorem}
\newtheorem{setting}{Setting}
\allowdisplaybreaks

\title{Deep Nonparametric Estimation of Intrinsic Data Structures  by Chart Autoencoders: Generalization Error and Robustness}

\author{Hao Liu, Alex Havrilla, Rongjie Lai and Wenjing Liao \thanks{Hao Liu is affiliated with the Math department of Hong Kong Baptist University; Alex Havrilla and Wenjing Liao are affiliated with the School of Math at Georgia Tech; Rongjie Lai is affiliated with the Department of Mathematics at Purdue University. Rongjie Lai and Wenjing Liao are co-corresponding authors. Email: \text{haoliu@hkbu.edu.hk, lairj@purdue.edu, $\{$ahavrilla3, wliao60$\}$@gatech.edu}. This research is partially supported by HKBU 179356, NSFC 12201530, HKRGC ECS 22302123, NSF DMS--2012652, NSF DMS-2145167 and NSF DMS--2134168.}}\date{}
\newcommand{\commentout}[1]{}

\begin{document}

\maketitle

\begin{abstract}

Autoencoders have demonstrated remarkable success in learning  low-dimensional latent features of high-dimensional data across various applications. Assuming that data are sampled near a low-dimensional manifold, we employ chart autoencoders, which encode data into low-dimensional latent features on a collection of charts, preserving the topology and geometry of the data manifold.
Our paper establishes statistical guarantees on the generalization error of chart autoencoders, and we demonstrate their denoising capabilities by considering $n$ noisy training samples, along with their noise-free counterparts, on a $d$-dimensional manifold. By training autoencoders, we show that chart autoencoders can effectively denoise the input data with normal noise. 
We prove that, under proper network architectures, chart autoencoders achieve a squared generalization error in the order of $\displaystyle n^{-\frac{2}{d+2}}\log^4 n$, which depends on the intrinsic dimension of the manifold and only weakly depends on the ambient dimension and noise level. We further extend our theory on data with noise containing both normal and tangential components, where chart autoencoders still exhibit a denoising effect for the normal component.
As a special case, our theory also applies to classical autoencoders, as long as the data manifold has a global parametrization. Our results provide a solid theoretical foundation for the effectiveness of autoencoders, which is further validated through several numerical experiments.
\end{abstract}

\section{Introduction}
High-dimensional data arise in many real-world machine learning problems, presenting new difficulties for both researchers and practitioners. For example, the ImageNet classification task \citep{deng2009imagenet} involves data points with 150,528 dimensions, derived from images of size $224\times224\times3$. Similarly, the MS-COCO object detection task \citep{lin2014microsoft} tackles data points with 921600 dimensions, stemming from images of size $480\times640\times3$. The well-known phenomenon of the curse of dimensionality states that, in many statistical learning and inference tasks, the required sample size for training must grow exponentially with respect to the dimensionality of the data, unless further assumptions are made. Due to this curse, directly working with high-dimensional datasets can result in subpar performance for many machine learning methods.

Fortunately, many real-world data are embedded in a high-dimensional space while exhibiting low-dimensional structures
due to local regularities, global symmetries, or repetitive patterns. 
It has been shown in \citet{pope2020intrinsic} that many benchmark datasets such as MNIST, CIFAR-10, MS-COCO and ImageNet have low intrinsic dimensions. 
In literature, a well-known mathematical model to capture such low-dimensional geometric structures in datasets is the manifold model, where data is assumed to be sampled on or near a low-dimensional manifold \citep{tenenbaum290global,roweis2000nonlinear,fefferman2016testing}.  
A series of works on manifold learning have been effective on nonlinear dimension reduction of data, including IsoMap \citep{tenenbaum290global}, Locally Linear Embedding \citep{roweis2000nonlinear,zhang2006mlle}, Laplacian Eigenmap \citep{belkin2003laplacian}, Diffusion map \citep{coifman2005geometric}, t-SNE \citep{van2008visualizing}, Geometric Multi-Resolution Analysis \citep{allard2012multi, liao2019adaptive} and many others \citep{aamari2019nonasymptotic}. 
As extensions, the noisy manifold setting has been studied in \citep{maggioni2016multiscale,genovese2012minimax,genovese2012manifold,puchkin2022structure}.

In recent years, deep learning has made significant successes on various machine learning tasks with high-dimensional data sets. 
Unlike traditional manifold learning methods which estimate the data manifold first and then perform statistical inference on the manifold, it is a common belief that deep neural networks can automatically capture the low-dimensional structures of the data manifold and utilize them for statistical inference. 
In order to justify the performance of deep neural networks, many mathematical theories have been established on function approximation \citep{hornik1989multilayer,yarotsky2017error,shaham2018provable,schmidt2019deep,shen2019deep,chen2019efficient,cloninger2021deep,montanelli2020error,liu2022benefits,liu2022adaptive}, regression \citep{chui2018deep,chen2019nonparametric,nakada2020adaptive,he2023side}, classification \citep{liu2021besov}, operator learning \citep{liu2022deep} and causal inference on a low-dimensional manifold \citep{chen2020doubly}. 
 In many of these works, a proper network architecture is constructed to approximate certain class of functions supported on a manifold. Regression, classification, operator learning and causal inference are further achieved with the constructed network architecture.  The sample complexity critically depends on the intrinsic dimension of the manifold and only weakly depends on the ambient dimension.

Autoencoder is a special designed deep learning method to effectively learn low-dimensional features of data \citep{bourlard1988auto,kramer1991nonlinear, hinton1993autoencoders, liou2014autoencoder}. 
The conventional autoencoder consists of two subnetworks, an encoder and a decoder. The encoder transforms the high-dimensional input data into a lower-dimensional latent representation, capturing the intrinsic parameters of the data in a compact form. The decoder then maps these latent features to reconstruct the original input in the high-dimensional space.
Inspired by the traditional autoencoder \citep{bengio2006greedy,ranzato2006efficient}, many variants of autoencoder have been proposed.
The most well-known variant of autoencoders is Variational Auto-Encoder (VAE) \citep{kingma2013auto,kingma2019introduction,rezende2014stochastic}, which introduces a prior distribution in the latent space as a regularizer. This regularization ensures a better control of the distribution of latent features and helps to avoid overfitting. Recently, the excess risk of VAE via empirical Bayes estimation was analyzed in \citet{tang2021empirical}.
The Denoising Auto-Encoder (DAE)\citep{vincent2008extracting,bengio2009learning} was proposed to denoise the input data in the process of feature extraction. By intentionally corrupting the input training data by noise, DAE has a denoising effect on the noisy test data and therefore has improved the robustness over the traditional autoencoders.

Although autoencoders have demonstrated great success in feature extraction and dimension reduction, its mathematical and statistical theories are still very limited. More importantly, the aforementioned autoencoders aim to globally map the data manifold to a subset in $\RR^{d}$ where $d$ is the intrinsic dimension of the manifold. However, a global mapping may not always exist for manifolds with nontrivial geometry and topology. To address this issue, \citet{schonsheck2019chart} showed that conventional auto encoders using a flat Eucildean space can not represent manifolds with nontrivial topology, thus introduced a Chart Auto-Encoder (CAE) to capture local latent features. Instead of using a global mapping, CAE uses a collection of open sets to cover the manifold where each set is associated with a local mapping. Their numerical experiments have demonstrated that CAE can preserve the  geometry and topology of data manifolds.
They also obtained an approximation error of CAE in the noise-free setting. Specifically,  \citet{schonsheck2019chart} constructed an encoder $\scrE$ and a decoder $\scrD$ that can optimize the empirical loss with the approximation error satisfying $\sup_{\bv \in \cM}\|\vb-\scrD\circ\scrE(\vb)\|_2 \le \varepsilon$, where $\cM$ represents the low-dimensional manifold. In a recent work \citep{schonsheck2022semi}, CAE has been extended to semi-supervised manifold learning and has demonstrated great performances in differentiating data on nearby but disjoint manifolds.

In this paper, our focus is on CAE and we aim to extend previous results in two ways. Firstly, we establish statistical guarantees on the generalization error for the trained encoders and decoders, which are given by the global minimizer of the empirical loss. Secondly, our analysis considers data sampled on a manifold corrupted by noise, which is a more practical scenario. To the best of our knowledge, this type of analysis has not been conducted previously. 
The generalization error analysis is crucial in understanding the sample complexity of autoencoders. Additionally, the inclusion of noise in the error analysis is significant as it allows us to examine the impact of noise on CAE.

We briefly summerize our results as follows. To demonstrate the robustness of CAE, we allow for data sampled on a manifold corrupted by noise.  
Namely, we assume $n$ pairs of clean and noisy data for training, where the clean data are sampled on a $d$-dimensional manifold, and the noisy data are perturbed from the clean data by noise.  This setting is practically meaningful and has been considered in DAE \citep{vincent2008extracting,bengio2009learning} and multi-fidelity simulations \citep{koutsourelakis2009accurate,biehler2015towards,parussini2017multi}. 
We show that CAE results in an encoder $\widehat{\scrE}$ and a decoder $\widehat{\scrD}$ that have a denoising effect for the normal noise. That is, for any noisy test data $\xb$, the output of $\widehat{\scrD}\circ\widehat{\scrE}(\xb)$ is close to its clean counterpart $\pi(\xb)$, which is the orthogonal projection of $\xb$ onto the manifold $\cM$. Our results, as summarized in Theorem \ref{thm.multi}, can be stated informally as follows:
\begin{theorem*}[Informal]
	Let $\cM$ be a $d$-dimensional compact smooth Riemannian manifold isometrically embedded in $\RR^D$ with reach $\tau>0$. Given a fixed noise level $q \in [0,\tau)$, we consider a training data set $\cS=\{(\bx_i,\vb_i)\}_{i=1}^n$ where the $\vb_i$'s are i.i.d. samples from a probability measure on $\cM$, and  $\xb_i=\vb_i+\wb_i$'s are perturbed from  the $\vb_i$'s with independent random noise $\wb_i \in T_{\vb_i}^{\perp}{\cM}$ (the normal space of $\cM$ at $\vb_i$) whose distribution satisfies $\|\wb\|_2\leq q$. 
 We denote the distribution of all $\xb_i$ by $\gamma$.  
Using proper network architectures,  the encoder ${\scrE}:\RR^D\rightarrow \RR^{O(d)}$ and the decoder ${\scrD}:\RR^{O(d)}\rightarrow \RR^D$, we solve the empirical risk minimization problem in (\ref{eq.loss}) to obtain the global minimizer
  $\widehat{\scrE}$ and $\widehat{\scrD}$. Then the expected generalization error of CAE satisfies
	\begin{align}
		\EE_{\cS}\EE_{\xb\sim \gamma} \| \widehat{\scrD}\circ\widehat{\scrE}(\xb) -\pi(\xb)\|_2^2\leq CD^2\log^3Dn^{-\frac{2}{d+2}}\log^4n
	\end{align}
where $C$ is a constant independent of $n$ and $D$. 
 \end{theorem*}

This theorem highlights the robustness and effectiveness of CAE in learning the underlying data manifold. More specifically, with increasing sample size $n$, the generalization error converges to zero at a fast rate and its exponent depends only on the intrinsic dimension $d$, not the ambient dimension $D$. The theorem also shows that autoencoders have a strong denoising ability when dealing with noise on the normal directions, as the error approaches zero as $n$ increases. The latent feature in every chart has a dimension of $d$, with the number of charts being dependent on the complexity of the manifold $\cM$. In special cases where the manifold is globally homeomorphic to a subset of $\RR^d$, this result can be applied to conventional autoencoders as described in Section \ref{sec.single}.

Besides the case of normal noise in the aforementioned theorem, we also consider a general setting where the noise contains both normal and tangential components. In Theorem \ref{thm.gaussian}, we prove that CAE can denoise the normal component of the noise.  Specifically, the squared generalization error is upper bounded by $$C(D^2\log^3D) n^{-\frac{2}{2+d}}\log^4 n +C_1\sigma^2,$$ where $\sigma^2$ is the second moment of the tangential component of the noise. Our result is consistent with the existing works in manifold learning \citep{genovese2012minimax,puchkin2022structure} which demonstrates that denoising is possible for normal noise but impossible for tangential noise. A detailed explanation is given at the end of Section \ref{sec.multi}.

The rest of the paper is organized as follows: In Section \ref{sec.preliminary}, we introduce background related to manifolds and neural networks to be used in this paper. In Section \ref{sec.mainresults}, we present our problem setting and main results including single chart case, multi-chart case and extension to general noise. We defer theoretical proof in Section \ref{sec.proofofmain}. We validate our network architectures and theories by several experiments in Section \ref{sec.experiments}. We conclude the paper in Section \ref{sec.conclusion}

\noindent\textbf{Notation}: We use lower-case letters to denote scalars, lower-case bold letters to denote vectors, upper-case letters to denote matrices and constants, calligraphic letters to denote manifolds, sets and function classes.
For a vector valued function $\fb=[f_1,...,f_d]^{\top}$ defined on $\Omega$, we let $\|\fb\|_{L^{\infty,\infty}} := \sup_{\xb\in \Omega}\max_{k} |f_k(\xb)|$.

\section{Preliminary}
\label{sec.preliminary}

In this section, we briefly introduce the preliminaries on manifolds and neural networks to be used in this paper.

\subsection{Manifolds}
We first introduce some definitions and notations about manifolds. More details can be found in \cite{tu2011manifolds,lee2006riemannian}. Let $\cM$ be a $d$-dimensional Riemannian manifold isometrically embedded in $\RR^D$. A chart of $\cM$ defines a local  neighborhood and coordinates on $\cM$.
\begin{definition}
	A chart of $\cM$ is a pair $(U,\phi)$ where $U\subset \cM$ is an open set, and $\phi:U\rightarrow \RR^d$ is a homeomorphism, i.e., $\phi$ is bijective and both $\phi$ and $\phi^{-1}$ are continuous.
	A $C^s$ atlas of $\cM$ is a collection of charts $\{(U_k,\phi_k)\}_{k\in \cK}$ which satisfies $\cup_{\alpha\in\cK} U_k=\cM$, and are pairwise $C^s$ compatible:
	$$
	\phi_{k_1}\circ \phi_{k_2}^{-1}: \phi_{k_2}(U_{k_1}\cap U_{k_2}) \rightarrow \phi_{k_1}(U_{k_1}\cap U_{k_2}) \quad \mbox{ and } \quad \phi_{k_2}\circ\phi_{k_1}^{-1}: \phi_{k_1}(U_{k_1}\cap U_{k_2}) \rightarrow \phi_{k_2}(U_{k_1}\cap U_{k_2})
	$$
	are both $C^s$ for any $k_1,k_2\in \cK$. An atlas is called finite if it contains finite many charts. Here $C^s$ denotes the space of functions with continuous derivatives up to $s$ order.
	
\end{definition}
A smooth manifold is a manifold with a $C^{\infty}$ atlas. Commonly used smooth manifolds include the Euclidean space, torus, sphere and Grassmannian. $C^s$ functions on a smooth manifold $\cM$ can be defined as follows:
\begin{definition}[$C^s$ functions on a smooth manifold]
	Let $\cM$ be a smooth manifold and $f:\cM\rightarrow \RR$ be a function on $\cM$. The function $f$ is a $C^s$ function on $\cM$ if for every chart $(U,\phi)$ of $\cM$, the function $f\circ\phi^{-1}: \phi(U)\rightarrow \RR$ is a $C^s$ function.
\end{definition}
We next define the $C^{\infty}$ partition of unity of $\cM$. 
\begin{definition}[Partition of unity]
	A $C^{\infty}$ partition of unity of a manifold $\cM$ is a collection of $C^{\infty}$ functions $\{\rho_k\}_{k\in \cK} $ with $\rho_k:\cM\rightarrow [0,1]$ such that for any $\xb\in \cM$,
	\begin{enumerate}
		\item there is a neighbourhood of $\xb$ where only a finite number of the functions in $\{\rho_k\}_{k\in \cK} $  are nonzero, and
		\item $\sum_{k\in \cK} \rho_k(\xb)=1$.
	\end{enumerate}
\end{definition}

We say an open cover is locally finite if every $\vb\in\cM$ has a neighborhood that intersects with a finite number of sets in the cover. It is well-known that for any locally finite cover of $\cM$, a $C^{\infty}$ partition of unity that subordinates to this cover always exists \cite[Chapter 2, Theorem 15]{spivak1975comprehensive}.

Reach is an important quantity of a manifold that  is related to curvature. For any $\xb\in\RR^D$, we write $d(\xb,\cM):=\inf_{\vb\in\cM}\|\xb-\vb\|_2$ the distance from $\xb$ to $\cM$. Reach is defined as follows:
\begin{definition}[Reach \citep{federer1959curvature,niyogi2008finding}]
	\label{def.reach}The reach of $\cM$ is defined as
\begin{align}
	\tau=\inf_{\vb\in\cM} \inf_{\xb\in G} \|\xb-\vb\|_2.
\end{align}
where $G= \left\{\xb\in \RR^D: \exists \mbox{ distinct } \pb,\qb\in \cM \mbox{ such that } d(\xb,\cM)=\|\xb-\pb\|_2=\|\xb-\qb\|_2\right\}$ is the medial axis of $\cM$.
\end{definition}
Roughly speaking, a manifold with a small reach can \lq\lq bend\rq\rq\ faster than that with a large reach.  For example, a plane has a reach equal to infinity. A hyper-sphere with radius $r$ has a reach $r$.  We illustrate manifolds with a large reach and small reach in Figure \ref{fig.reach}.
\begin{figure}
	\centering
	\includegraphics[width=0.6\textwidth]{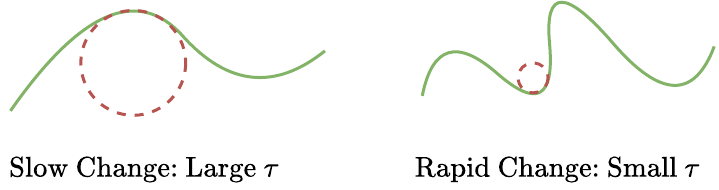}
	\caption{Illustration of manifolds with large and small reach.}
	\label{fig.reach}
\end{figure}

We denote the reach of $\cM$ by $\tau$, the tangent plane of $\cM$ at $\vb\in \cM$ by $T_{\vb}\cM$ and its orthogonal space by $T_\vb^{\perp}{\cM}$. Define the $q$--neighborhood of $\cM$ by
\begin{align}
	\cM(q)=\{\xb\in \RR^{D}: \inf_{\vb\in\cM} \|\xb-\vb\|_2\leq q\}.
\end{align} 
When $q<\tau$, by the property of reach, every $\xb\in \cM(q)$ has a unique decomposition 
\begin{align}
	\xb= \pi(\xb) +\wb
\end{align}
where $\pi(\xb)= \argmin_{\vb\in\cM } \|\vb-\xb\|_2$  and $\wb\in T_\vb^{\perp}{\cM}$ \citep{niyogi2008finding,cloninger2021deep}. 

\subsection{Neural networks}
In this paper, we consider feedforward neural networks (FNN) with the rectified linear unit $\ReLU(a)=\max\{a,0\}$. 
An FNN with $L$ layers is defined as
\begin{align}
	f(\xb)=W_L\cdot\ReLU\left( W_{L-1}\cdots \ReLU(W_1\xb+\bbb_1)+ \cdots +\bbb_{L-1}\right)+\bbb_L,
	\label{eq.ReLU}
\end{align}
where the $W_i$'s are weight matrices, the $\bbb_i$'s are bias vectors, and $\ReLU$ is applied element-wisely. We define a class of neural networks with inputs in $\RR^D$ and outputs in $\RR^d$ as
\begin{align*}
	\cF(D,d;L,p,K,\kappa,R) = &\{f:\RR^D\rightarrow\RR^d ~|~ f\mbox{ has the form of (\ref{eq.ReLU}) with } L \mbox{ layers and width bounded by } p, \\
	&\qquad \|f\|_{\infty}\leq R, \sum_{i=1}^L \|W_i\|_0+\|\bbb_i\|_0\leq K,\\
	&\qquad \|W_i\|_{\infty, \infty}\leq \kappa, \|\bbb_i\|_{\infty}\leq \kappa \mbox{ for } i=1,...,L\},
\end{align*}
where $\norm{H}_{\infty, \infty} = \max_{i, j} |H_{ij}|$ for a matrix $H$ and $\|\cdot\|_0$ denotes the number of non-zero elements of its argument. Above, the width of a network is the largest output dimension among all layers.

\section{Main results}
\label{sec.mainresults}

\subsection{Problem setup for bounded normal noise}
We consider the noisy setting where training data contain $n$ pairs of clean and noisy data: 

\begin{setting}\label{setting}
Let  $\cM$ be a $d$-dimensional compact smooth Riemannian manifold isometrically embedded in $\RR^D$ with reach $\tau$. Given a fixed noise level $q \in [0,\tau)$, we consider a training data set $\cS=\{(\bx_i,\vb_i)\}_{i=1}^n$ where the $\vb_i$'s are i.i.d. samples from a probability measure on $\cM$, and  the $\xb_i$'s are perturbed from  the $\vb_i$'s according to the model such that
	\begin{align}
		\xb=\vb+\wb
	\end{align}
	where $\wb \in T_\vb^{\perp}{\cM}$ (the normal space of $\cM$ at $\vb$ ) is a random vector satisfying $\|\wb\|_2\leq q$. 
 We denote the distribution of $\xb$ by $\gamma$. In particular, we have $\xb_i=\vb_i+\wb_i$, where the $\wb_i$'s are  independent.

\end{setting}

Setting \ref{setting} has two important implications:
\begin{enumerate}[(1)]
	\item[(i)]  $\cM$ is bounded: there exists a constant $B>0$ such that for any $\xb\in\cM(q)$,
	\begin{align}
		\|\xb\|_{\infty}\leq B.
	\end{align}
	\item[(ii)]  $\cM$ has a positive reach \citep[Proposition 14]{thale200850}, denoted by $\tau>0$. 
\end{enumerate}

The $\vb_i$'s in Setting \ref{setting} represent the noise-free training data, and the $\xb_i$'s are the noisy data perturbed by the normal noise $\wb_i$'s.
This noisy setting shares some similarity to the Denoising Auto-Encoder (DAE) \citep{vincent2008extracting,bengio2009learning}  and multifidelity simulations \citep{koutsourelakis2009accurate,biehler2015towards,parussini2017multi}.  DAE is widely used in image processing to train autoencoders with a denoising effect. In the DAE setting, one has clean samples, and then manually adds noise to the clean samples to obtain noisy samples.
                During training, the noisy samples are taken as the inputs and the clean samples are the outputs, such that the autoencoder is trained to denoise the noisy samples. 
In uncertainty quantification and prediction of random fields, it is expensive to simulate high-fidelity solutions. A popular strategy is to use a cheaper low-fidelity simulation as a surrogate and then a correction step is applied to modify the surrogate towards high-fidelity data. The correction operations are determined using both low-fidelity and high-fidelity data. Such a strategy is similar to our Setting \ref{setting}: one can take the high-fidelity data as noise-free data and low-fidelity data as  noisy data.

We first consider normal noise on the manifold $\cM$ in Setting \ref{setting}.
Given a training data set  $\cS=\{(\xb_i,\vb_i)\}_{i=1}^n$, our goal is to theoretically analyze how the manifold structure of data can be learned based an encoder $\widehat{\scrE}: \cM(q)\rightarrow \RR^{O(d)}$ and the corresponding decoder $\widehat{\scrD}:\RR^{O(d)}\rightarrow \RR^D$ by minimizing the empirical mean squared loss
\begin{align}
	(\widehat{\scrD},\widehat{\scrE})=\argmin_{\scrD\in \cF_{\rm NN}^{\scrD},\scrE\in \cF_{\rm NN}^{\scrE}}\frac{1}{n}\sum_{i=1}^n \|\vb_i-\scrD\circ \scrE(\xb_i)\|_{2}^2,
	\label{eq.loss}
\end{align}
for properly designed network architectures $\cF_{\rm NN}^{\scrE}$ and $\cF_{\rm NN}^{\scrD}$.
We evaluate the performance of $(\widehat{\scrD},\widehat{\scrE})$ 
through the squared generalization error 
\begin{align}
	\EE_{\cS}\EE_{\xb\sim \gamma} \| \widehat{\scrD}\circ\widehat{\scrE}(\xb)- \pi(\xb)\|_2^2.
	\label{eq.error}	
\end{align}
at a noisy test point $\xb$ sampled from the same distribution $\gamma$ as the training data.
This paper establishes upper bounds on the squared generalization error of CAE with properly chosen network architectures. We first consider the single-chart case in Section \ref{sec.single} where $\cM$ is globally homeomorphic to a subset of $\RR^d$. The general multi-chart case is studied in Section \ref{sec.multi}. In Section \ref{sec.gaussian}, we will study a more general setting that allows high-dimensional  noise in the ambient space, under Setting \ref{setting.gaussian}.

\subsection{Single-chart case} \label{sec.single}

We start from a simple case where $\cM$ has a global low-dimensional parametrization. In other words, data on $\cM$ can be encoded to a $d$-dimensional latent feature through a global mapping.
\begin{assumption}[Single--chart case]\label{assum.single}
	Assume $\cM$ has a global $d$--dimensional parameterization:
		There exist $\Lambda>0$ and smooth maps $\fb:\cM\rightarrow [-\Lambda,\Lambda]^d$ and $\gb:[-\Lambda,\Lambda]^d\rightarrow \cM$ such that 		\begin{align}
			\vb=\gb\circ \fb(\vb).
		\end{align}
		for any $\vb\in\cM$. 

\end{assumption}

Assumption \ref{assum.single} implies that there exists an atlas of $\cM$ consisting of only one chart $(\cM, \fb)$. This single-chart case serves as the mathematical model of autoencoders where one can learn a global low-dimensional representation of data without losing much information.

Our first result gives an upper bound on the generalization error (\ref{eq.error}) with properly chosen network architectures.
\begin{theorem}\label{thm.single}
	In Setting \ref{setting}, suppose Assumption \ref{assum.single} holds. Let $\widehat{\scrE},\widehat{\scrD}$ be a global minimizer in (\ref{eq.loss}) with the network classes $\cF_{\rm NN}^{\scrE}=\cF(D,d;L_\scrE,p_\scrE,K_\scrE,\kappa_\scrE,R_\scrE)$ and $\cF_{\rm NN}^{\scrD}=\cF(d,D;L_\scrD,p_\scrD,K_\scrD,\kappa_\scrD,R_\scrD)$ where
	\begin{gather}
		L_{\scrE}=O\left(\log^2 n+\log D\right), \ p_{\scrE}=O\left(D n^{\frac{d}{d+2}}\right), \ K_{\scrE}=O\left(Dn^{\frac{d}{d+2}}\log^2 n+ D\log D\right), \nonumber\\
		\kappa_{\scrE}=O\left(n^{\frac{2}{d+2}}\right),\ R_{\scrE}=\Lambda,\\
		L_{\scrD}=O\left(\log n\right), \ p_{\scrD}=O\left(n^{\frac{d}{d+2}}\right), \ K_{\scrD}=O\left(n^{\frac{d}{d+2}}\log^2 n\right),\ \kappa_{\scrD}=O\left(n^{\frac{1}{d+2}}\right), \ R_{\scrD}=B.
	\end{gather}  
Then, we have the following upper bound of the squared generalization error
\begin{align}
	\EE_{\cS}\EE_{\xb\sim \gamma} \| \widehat{\scrD}\circ\widehat{\scrE}(\xb) -\pi(\xb)\|_2^2\leq C\left(D^2\log^2 D\right)n^{-\frac{2}{d+2}}\log^4 n
\end{align}
for some constant $C$ depending on $d,B,\tau,q$, the Lipschitz constant of $\fb$ and $\gb$, and the volume of $\cM$.
The constants hidden in the $O$ depend on $d,B,\Lambda,\tau,q$, the volume of $\cM$ and the Lipschitz constant of $\fb$ and $\gb$.
\end{theorem}
We defer the detailed proof of Theorem \ref{thm.single} in Section \ref{proof.single}. Theorem \ref{thm.single} has several implications: 
\begin{itemize}
	\item[(i)] {\bf Fast convergence of the squared generalization error and the denoising effect:} When the network architectures are properly set, we can learn an autoencoder and the corresponding decoder so that the squared generalization error converges at a fast rate in the order of $\displaystyle n^{-\frac{2}{d+2}}\log^4 n$. Such a rate crucially depends on the intrinsic dimension $d$ instead of the ambient dimension $D$, and therefore mitigates the curse of ambient space dimensionality. In addition, the error bound also suggests that the autoencoder has a denoising effect as the network output $\widehat{\scrD}\circ\widehat{\scrE}(\xb) $ converges to its clean counterpart $\pi(\xb)$ as $n$ increases.
	\item[(ii)] {\bf Geometric representation of data:}  When the manifold $\cM$ has a global $d$-dimensional parameterization, the autoencoder $\widehat{\scrE}$ outputs a $d$-dimensional latent feature, which serves as a geometric representation of the high-dimensional input $\xb$.

	\item[(iii)] {\bf Network size:} The network size critically depends on $d$, and weakly depends on $D$.
\end{itemize}

\begin{remark}
	We remark that the constant hidden in the upper bound in Theorem \ref{thm.single} (and for the constant in Theorem \ref{thm.multi}) depends on $1/(\tau-q)$: as $q$ gets closer to $\tau$, the constant factor becomes larger. This is easy to understand: if $q$ is very close to $\tau$, some data are close to the medial axis of $\cM$ (see Definition \ref{def.reach}). Assume $\xb=\vb+\wb$ for $\vb\in\cM$ and $\wb\in T_\vb^{\perp}{\cM}$ with $\|\wb\|_2$ being very close to $\tau$. Then there exists $\vb'(\neq \vb)$ on $\cM$ such that $\|\vb'-\xb\|_2$ is very close to $\|\vb-\xb\|_2$. Thus, a small perturbation in $\xb$ might lead to a big change in $\pi(\xb)$, which makes the projection unstable.
\end{remark}

We briefly introduce the proof idea of Theorem \ref{thm.single} in four steps. 

	\noindent{\bf Step 1: Decomposing the error.} We decompose the squared generalization error (\ref{eq.error}) into a squared bias term and a variance term. The bias term captures the network's approximation error and the variance term captures the stochastic error.
	
	\noindent{\bf Step 2: Bounding the bias term.} To derive an upper bound of the bias term, we first define the oracle encoder and decoder. According to Assumption \ref{assum.single}, $\fb$ is an encoder and $\gb$ is a decoder of $\cM$. However, since the input data is in $\cM(q)$, we cannot directly use $\fb$ as the encoder since $\fb$ is defined on $\cM$. Utilizing the projection operator $\pi$, we define the oracle encoder as $\scrE=\fb\circ\pi$, and simply define the oracle decoder as $\scrD=\gb$. Based on \citet{cloninger2021deep}, we design the  encoder network
	to approximate the oracle encoder $\fb\circ\pi$. The decoder network 
	is designed to approximate $\gb$. Our network architecture is illustrated in Figure \ref{fig.network.thm1}. Based on our network  construction, we derive an upper bound of the bias term showing that our encoder and decoder networks can approximate the oracles $\scrE$ and $\scrD$ respectively to an arbitrary accuracy $\varepsilon$ (see Lemma \ref{lem.single.approx}). 
	
	\noindent {\bf Step 3: Bounding the variance term.} The upper bound for the variance term is derived using metric entropy arguments \cite{vaart1996weak,gyorfi2002distribution}, which depends on the network size (see Lemma \ref{lem.single.T2}).
	
	\noindent {\bf Step 4: Putting the upper bound for both terms together.} We finally put the upper bounds of the squared bias and variance term together. After balancing the approximation error and the network size, we can prove Theorem \ref{thm.single}.

\begin{figure}
	\centering
		\includegraphics[scale=0.6]{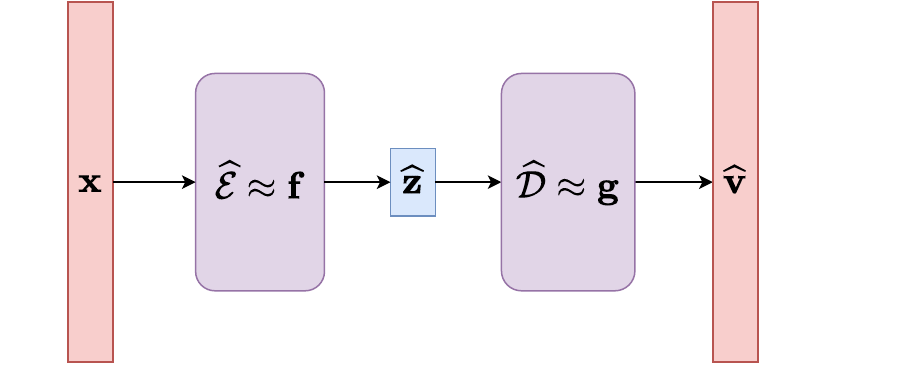}\\
	\caption{An illustration of the network architectures in Theorem \ref{thm.single}, where $\widehat{\scrE}$ is the estimated encoder, $\widehat{\scrD}$ is the estimated decoder, $\widehat{\zb}$ is the encoded latent feature, $\widehat{\vb}$ is the decoded vector. } 
	\label{fig.network.thm1}
\end{figure}
\subsection{Multi-chart case}\label{sec.multi}
We next consider a more general case where the manifold has a complicated topology requiring multiple charts in an atlas. Consider an atlas $\{(U_k,\phi_k)\}_{k=1}^{C_{\cM}}$ of $\cM$ so that 
each $U_k$ is a local neighborhood on the manifold homeomorphic to a subset of $\RR^d$. Here $C_{\cM}$ denotes the number of charts in this atlas. In this atlas, each $U_k$ has a $d$-dimensional parametrization of $\cM$ and data can be locally represented by $d$-dimensional latent features.
In this general case, we prove the following upper bound of the squared generalization error (\ref{eq.error}) for CAE.

\begin{theorem}\label{thm.multi}
	Consider Setting \ref{setting}.    Let $\widehat{\scrE},\widehat{\scrD}$ be a global minimizer of (\ref{eq.loss}) with the  network classes $\cF_{\rm NN}^{\scrE}=\cF(D,C_{\cM}(d+1);L_\scrE,p_\scrE,K_\scrE,\kappa_\scrE,R_\scrE)$ and $\cF_{\rm NN}^{\scrD}=\cF(C_{\cM}(d+1),D;L_\scrD,p_\scrD,K_\scrD,\kappa_\scrD,R_\scrD)$ where $C_{\cM}=O((d\log d)(4/\tau)^d)$,
	\begin{gather}
		L_{\scrE}=O(\log^2 n+\log D), \ p_{\scrE}=O(Dn^{\frac{d}{d+2}}), \ K_{\scrE}=O((D\log D)n^{\frac{d}{d+2}}\log^2 n), \nonumber\\
		\kappa_{\scrE}=O(n^{\frac{2}{d+2}}), \ R_{\scrE}=\max\{\tau/4,1\},
		\label{eq.multi.E.para}\\
		L_{\scrD}=O(\log^2 n+\log D), \ p_{\scrD}=O(Dn^{\frac{d}{d+2}}), \ K_{\scrD}=O(Dn^{\frac{d}{d+2}} \log^2 n +D\log D), \nonumber\\
		 \kappa_{\scrD}=O(n^{\frac{1}{d+2}}), \ R_{\scrD}=B.
		\label{eq.multi.D.para}
	\end{gather} 
Then, the we have the following upper bound of the squared generalization error
	\begin{align}
		\EE_{\cS}\EE_{\xb\sim \gamma} \| \widehat{\scrD}\circ\widehat{\scrE}(\xb) -\pi(\xb)\|_2^2\leq C(D^2\log^3D)n^{-\frac{2}{d+2}}\log^4n
	\end{align}
for some constant $C$ depending on $d,\tau,q,B,C_{\cM}$ and the volume  of $\cM$. 
\end{theorem}

\begin{figure}
	\centering
		\includegraphics[width=0.8\textwidth]{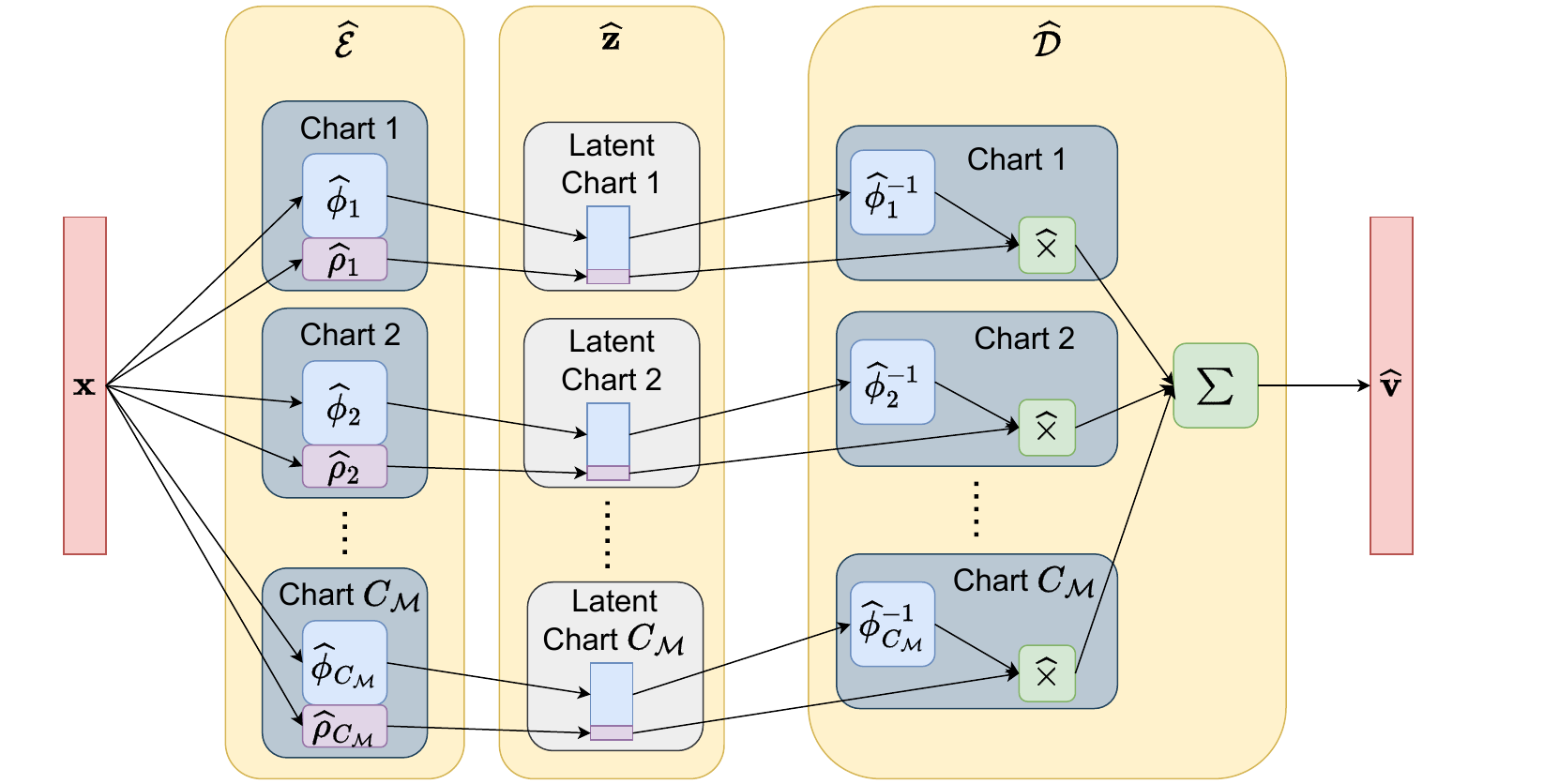}
	\caption{Illustration of the network architectures in Theorem \ref{thm.multi}. Here $\widehat{\scrE}$ is the estimated encoder, $\widehat{\scrD}$ is the estimated decoder.  The output of $\widehat{\scrE}$ gives rise to the latent feature $\widehat{\zb}$, and  the final output of the decoder $\widehat{\scrD}$ gives rise to $\widehat{\vb}$ as an approximation of the clean sample $\vb$.}
	\label{fig.network}
\end{figure}

Theorem \ref{thm.multi} indicates that for a general smooth manifold, the squared generalization error converges in the order of $n^{-\frac{2}{d+2}}\log^4n$. 
In the case of multiple charts, the manifold has a more complicated structure than the single-chart case. Compared to the network architectures in Theorem \ref{thm.single}, the network architecture specified in Theorem \ref{thm.multi} has the following changes:
\begin{itemize}
	\item The output of encoder has dimension $C_{\cM}(d+1)$ instead of $d$. Such an increment in dimension is due to that $C_{\cM}$ charts are needed to cover complicated manifolds.
	\item The encoder network uses more parameters: The number of nonzero parameters $K$ has an additional factor $\log D$.
	\item The decoder network is deeper, wider and uses more parameters: The depth has an additional factor $\log n$, the width has an additional factor $D$, and the number of nonzero parameters has an additional factor $D$.
\end{itemize}
We remark that the number of charts $C_{\cM}=O((d\log d)(4/\tau)^d)$ occurs in the worst case scenario. If $\cM$ has some good properties so that fewer charts are needed to cover $\cM$, the result in Theorem \ref{thm.multi} holds by replacing $C_{\cM}$ by the actual number of charts needed.

 We defer the detailed proof of Theorem \ref{thm.multi} to Section \ref{proof.multi}.  The proof idea of Theorem \ref{thm.multi} is similar to that of Theorem \ref{thm.single}. We also decompose the squared generalization error into a squared bias term and a variance term. The bias term is controlled by neural network approximation theories of the oracles.
We briefly discuss the oracles and our network construction in Theorem \ref{thm.multi} here.
Let $\{U_j, \phi_j\}_{j=1}^{C_{\cM}}$ be an atlas of $\cM$ and $\{\rho_j\}_{j=1}^{C_{\cM}}$ be the partition of unity that subordinates to the atlas. For any $\xb\in \cM(q)$, we have
\begin{align}
	\vb=\pi(\xb)=\sum_{j=1}^{C_{\cM}}\left[ \phi_j^{-1}\circ\phi_j\circ\pi(\xb)\right]\times \left[ \rho_j\circ\pi(\xb) \right].
	\label{eq.x.decomposition}
\end{align}
If $\pi(\xb)$ is on the $j$th chart, \eqref{eq.x.decomposition} gives rise to an encoding of $\xb \in \cM(q)$ to the local coordinate $\phi_j\circ\pi(\xb) \in \RR^d$ and the partition of unity value $\rho_j\circ\pi(x)\in \RR$. We consider $[(\phi_j\circ\pi)^{\top} \ \ \rho_j\circ\pi]^{\top}$ as the oracle encoder on the $j$th chart, and the collection of $\{[(\phi_j\circ\pi)^{\top} \ \ \rho_j\circ\pi]^{\top}\}_{j=1}^{C_{\cM}}$ as the global encoder. The latent feature on a single chart is of dimension $d+1$, and the latent feature on the whole manifold is of dimension $C_{\cM}(d+1)$ with $C_{\cM}$ charts. For any latent feature $\{\zb_j \in \RR^{d+1}\}_{j=1}^{C_{\cM}}$, we consider the oracle decoder $\sum_{j=1}^{C_{\cM}} \phi_j^{-1}((\zb_j)_{1:d}) \times (\zb_j)_{d+1}$ where $(\zb_j)_{1:d}$ and $(\zb_j)_{d+1}$ represent the first $d$ and the $(d+1)$th entries of $\zb_j$ respectively. To bound the bias in Theorem \ref{thm.multi}, we design neural networks to approximate the oracle encoder and decoder with an arbitrary accuracy $\varepsilon$ (see Lemma \ref{lem.multi.approx} and its proof). The overall architecture of the encoder and decoder is illustrated in Figure \ref{fig.network}. 

\commentout{
or realize each component in (\ref{eq.x.decomposition}) by neural networks. 
The encoder $\widehat{\scrE}$ consists of the $\phi_j\circ\pi(\xb)$'s approximation and $\rho_j\circ\pi$'s approximation by neural networks for $j = 1,\ldots,C_{\cM}$. The latent feature given by the encoder  includes the $(d+1)$-outputs from each chart, and therefore is of dimension $C_{\cM}(d+1)$ with $C_{\cM}$ charts. 
The decoder $\widehat{\scrD}$ consists of the  $\phi_j^{-1}$'s approximation for $j = 1,\ldots,C_{\cM}$. The decoder also includes the summation operation in (\ref{eq.x.decomposition}) and a neural network approximation of multiplication $\times$. 
}

\begin{table}[h]
	\centering
 \renewcommand{\arraystretch}{1.5}
	\begin{tabular}{c|c|c|c}
		\hline
		Reference & Method & Error measurements & Upper bound\\
		\hline
		 \citet{canas2012learning}& k-Means & $\EE_{\cS}\EE_{\xb\in\cM}[{\rm dist}^2(\xb,\widehat{\cM})]$ & $n^{-\frac{1}{d+2}}$\\
		 \hline
		 \citet{canas2012learning}& k-Flats & $\EE_{\cS}\EE_{\xb\in\cM}[{\rm dist}^2(\xb,\widehat{\cM})]$ & $n^{-\frac{2}{d+4}}$\\
		 \hline
		 \cite{liao2019adaptive} & Multiscale & $\EE_{\cS}\EE_{\xb\in\cM}[\|{\rm Proj}_{\widehat{\cM}}(\xb)-\xb\|^2]$ & $n^{-\frac{4}{d+2}}$\\
		 \hline
          \cite{schonsheck2019chart} & Chart Auto-Encoder & $\sup_{\xb\in\cM } \| \widetilde{\scrD}\circ\widetilde{\scrE}(\xb)- \xb\|_2^2$ & $n^{-\frac{1}{d+1}}$\\
          \hline 
		 Theorem \ref{thm.multi} & Chart Auto-Encoder &$\EE_{\cS}\EE_{\xb\sim \gamma} \| \widehat{\scrD}\circ\widehat{\scrE}(\vb)- \vb\|_2^2$ & $n^{-\frac{2}{d+2}}$\\
		 \hline
	\end{tabular}
\caption{Summary of the results on manifold learning with noise-free data. In the table, $\widehat{\cM}$ denotes the manifold learned from training data, ${\rm dist}(\xb,\cM)$ denotes the distance from $\xb$ to $\cM$, ${\rm Proj}_{\widehat{\cM}}(\xb)$ denotes the projection of $\xb$ onto $\widehat{\cM}$. }
\label{tab.noiseless}
\end{table}

We next discuss the connection between our results and some existing works.
Low-dimensional approximations of manifolds have been studied in the manifold learning literature, with classical methods, such as k-means and k-flats \citep{canas2012learning}, multiscale linear approximations \citep{maggioni2016multiscale,liao2019adaptive}. \citet{canas2012learning}, \citet{liao2019adaptive} and \cite{schonsheck2019chart} consider the noise-free setting, where training and test data are exactly located on a low-dimensional manifold. This is comparable to our noise-free setting with $q=0$. 
We summarize the upper bounds in these works and our Theorem \ref{thm.multi} in Table \ref{tab.noiseless}. While the rate from \citet{liao2019adaptive} is faster than ours, $O\left(n^{\frac{1}{d+2}}\right)$ local tangent planes are used to approximate $\cM$. In comparison, our Theorem \ref{thm.multi} requires a fixed number $C_{\cM}$ (at most $O((d\log d)(4/\tau)^d)$) local pieces (charts), which is independent of the sample size $n$. 
\cite{schonsheck2019chart} first considers a Chart Auto-Encoder where their analysis is for an approximation error. 
 Given $n$ data samples uniformly distributed on $\cM$, \citet{schonsheck2019chart} explicitly constructs the encoder $\widetilde{\scrE}$ and decoder $\widetilde{\scrD}$ and shows that with high probability, the constructed autoencoder gives rise to the approximation error in Table \ref{tab.noiseless}. 
Our analysis in this paper extends the theory to the noisy setting and establishes a statistical estimation theory with an improved rate of convergence on the mean squared generalization error, which is beyond the approximation error analysis.

Our noisy setting shares some similarities with \cite{genovese2012minimax} and \cite{puchkin2022structure}, which focuses on manifold learning from noisy data. 
\cite{genovese2012minimax} assumes that the training data are corrupted by normal noise. In their setting, the noise follow a uniform distribution along the manifold normal direction and only noisy data (without the clean counterparts) are given for training. The authors proved that the lower bound measured by Hausdorff distance is in the order of $n^{-\frac{2}{2+d}}$, while no efficient algorithm is proposed to achieve this error bound. Recently, \cite{puchkin2022structure} considers a more general distribution of noise (not restricted to normal noise) but assumes the noise magnitude decays with a certain rate as $n$ increases. In comparison, our work and \cite{genovese2012minimax} do not require the noise magnitude to decay  as $n$ increases. The great advantage of \cite{puchkin2022structure}, as well as \cite{genovese2012minimax}, is that only noisy data are required for training.
\cite{puchkin2022structure} also derived a lower bound  measured  by Hausdorff distance  in the order of $\displaystyle \frac{q^2b^2}{\tau^3}\vee \tau^{-1}\left( \frac{q^2\tau^2\log n}{n} \right)^{\frac{2}{d+4}}$, where $b$ denotes an upper bound of the magnitude of the tangential component of noise. The lower bound is different from the one in \cite{genovese2012minimax}, because the existence of tangential noise, even with very small magnitude. 
Our work requires both clean and noisy data for training, which is possible when the training data are well-controlled. Our goal is to establish a theoretical foundation for the widely used autoencoders.

\subsection{Extension to general noise with bounded normal components} \label{sec.gaussian}
We consider a more general setting in which the noise includes both normal and tangential components.
\begin{setting}\label{setting.gaussian}
	Let  $\cM$ be a $d$-dimensional compact smooth Riemannian manifold isometrically embedded in $\RR^D$ with reach $\tau$, Given a fixed noise level $q \in [0,\tau)$, we consider a training data set  $\cS=\{(\bx_i,\vb_i)\}_{i=1}^n$ where the $\vb_i$'s are i.i.d. samples from a probability measure on $\cM$, and the $\xb_i$'s are perturbed from the $\vb_i$'s according to the model such that
	\begin{align}
		\xb=\vb+\nbb
  \label{eq.data.gaussian}
	\end{align}
	where $\nbb \in \RR^D$ is a random vector satisfying 
	$$
	\|\rm{Proj}_{T_\vb^{\perp}{\cM}}(\nbb)\|_2\leq q,\quad \EE\left[\|\rm{Proj}_{T_{\vb}\cM}(\nbb)\|^2_2|\vb\right]\leq \sigma^2
	$$ 
	with $\sigma\geq 0$. Here $\rm{Proj}_{T_{\vb}\cM}(\nbb)$ and $\rm{Proj}_{T_\vb^{\perp}{\cM}}(\nbb)$ denote the orthogonal projections of $\nbb$ onto the tangent space $T_{\vb}\cM$ and the normal space $T_\vb^{\perp}{\cM}$ respectively. We denote the distribution of $\xb$ by $\gamma$.  In particular, we have $\xb_i=\vb_i+\nbb_i$, where the $\nbb_i$'s are independent.

\end{setting}
For any noise vector $\nbb$, we can decompose it into the normal component $\rm{Proj}_{T_\vb^{\perp}{\cM}}(\nbb)$ and the tangential component $\rm{Proj}_{T_{\vb}\cM}(\nbb)$:
$$\nbb = \rm{Proj}_{T_\vb^{\perp}{\cM}}(\nbb)+\rm{Proj}_{T_{\vb}\cM}(\nbb).$$
Setting \ref{setting.gaussian} requires the magnitude of the normal component to be bounded by $q$, and for any $\vb\in\cM$, the second moment of the tangential component is bounded by $\sigma^2$. In particular, if we further have $\EE\left[\rm{Proj}_{T_{\vb}\cM}(\nbb)|\vb\right]=0$, Setting \ref{setting.gaussian} implies for any $\vb\in\cM$, the tangential component has variance no larger than $\sigma^2$. 

Under Setting \ref{setting.gaussian}, the squared generalization error (\ref{eq.error}) is not  appropriate any more, since our goal is to recover $\vb$ and $\pi(\xb)$ is not necessarily equal to $\vb$. Instead, we consider the following squared generalization error
\begin{align*}
	\EE_{\cS}\EE_{\xb\sim\gamma} \| \widehat{\scrD}\circ\widehat{\scrE}(\xb) -\vb\|_2^2.
\end{align*}
We have the following upper bound on the squared generalization error
\begin{theorem}\label{thm.gaussian}
	Consider Setting \ref{setting.gaussian}.    Let $\widehat{\scrE},\widehat{\scrD}$ be a global minimizer of (\ref{eq.loss}) with the  network classes $\cF_{\rm NN}^{\scrE}=\cF(D,C_{\cM}(d+1);L_\scrE,p_\scrE,K_\scrE,\kappa_\scrE,R_\scrE)$ and $\cF_{\rm NN}^{\scrD}=\cF(C_{\cM}(d+1), D;L_\scrD,p_\scrD,K_\scrD,\kappa_\scrD,R_\scrD)$ where $C_{\cM}=O((d\log d)(4/\tau)^d)$,
	\begin{gather}
		L_{\scrE}=O(\log^2 n+\log D), \ p_{\scrE}=O(Dn^{\frac{d}{d+2}}), \ K_{\scrE}=O((D\log D)n^{\frac{d}{d+2}}\log^2 n), \nonumber\\
		\kappa_{\scrE}=O(n^{\frac{2}{d+2}}), \ R_{\scrE}=\max\{\tau/4,1\},
		\label{eq.gaussian.E.para}\\
		L_{\scrD}=O(\log^2 n+\log D), \ p_{\scrD}=O(Dn^{\frac{d}{d+2}}), \ K_{\scrD}=O(Dn^{\frac{d}{d+2}} \log^2 n +D\log D), \nonumber\\
		\kappa_{\scrD}=O(n^{\frac{1}{d+2}}), \ R_{\scrD}=B.
		\label{eq.gaussian.D.para}
	\end{gather} 
	We have 
	\begin{align}
		\EE_{\cS}\EE_{\xb\sim \gamma} \| \widehat{\scrD}\circ\widehat{\scrE}(\xb) -\vb\|_2^2\leq C(D^2\log^3D)n^{-\frac{2}{d+2}}\log^4n+C_1 \sigma^2
		\label{eq.gaussian.bound}
	\end{align}
	for some constant $C$ depending on $d,\tau,q,B,C_{\cM}$ and the volume  of $\cM$, and $C_1$ depending on $\tau,q$. The constant hidden in $O$ depends on $d,\tau,q,B,C_{\cM}$ and the volume  of $\cM$.
\end{theorem}

Theorem \ref{thm.gaussian} is proved in Section \ref{proof.gaussian}. Theorem \ref{thm.gaussian} is a straightforward extension of Theorem \ref{thm.multi} to Setting  \ref{setting.gaussian}  with general noise. The network architecture in Theorem \ref{thm.gaussian} has a similar size as that in Theorem \ref{thm.multi}. We summarize the network architectures specified in Theorem \ref{thm.single}, \ref{thm.multi} and \ref{thm.gaussian} in Table \ref{tab.architecture}. Compared to the upper bound in Theorem \ref{thm.multi}, Theorem \ref{thm.gaussian} has an additional term $C_1\sigma^2$ which comes from the tangential component of noise. If tangential noise exists, a given point $\xb$ may correspond to multiple (and  probably infinitely many) points on $\cM$. Therefore the generalization error cannot converge to 0 as the sample size increases. 
This fundamental difficulty of high-dimensional noise is also demonstrated by our numerical experiments. 

\begin{table}
	\centering
	\begin{tabular}{c|c|c|c}
		\hline
		& & Theorem \ref{thm.single} & Theorem \ref{thm.multi}, \ref{thm.gaussian}\\
		\hline
		\multirow{ 6}{*}{$\cF^{\scrE}_{\rm NN}$} & dim. of input, output & $D,\ d$ & $D,\ C_{\cM}(d+1)$\\
		\cline{2-4}
		&  $L_{\scrE}$ & $O\left(\log^2 n+\log D\right)$ & $O\left(\log^2 n+\log D\right)$\\
		\cline{2-4}
		&  $p_{\scrE}$ & $O\left(D n^{\frac{d}{d+2}}\right)$  & $O\left(Dn^{\frac{d}{d+2}}\right)$\\
		\cline{2-4}
		&  $K_{\scrE}$ & $O\left(Dn^{\frac{d}{d+2}}\log^2 n+ D\log D\right)$ & $O\left((D\log D)n^{\frac{d}{d+2}}\log^2 n\right)$\\
		\cline{2-4}
		&  $\kappa_{\scrE}$ & $O\left(n^{\frac{2}{d+2}}\right)$ & $O\left(n^{\frac{2}{d+2}}\right)$\\
		\cline{2-4}
		&  $R_{\scrE}$ & $\Lambda$ & $\max\{\tau/4,1\}$\\
		\hline
		\multirow{ 6}{*}{$\cF^{\scrD}_{\rm NN}$} & dim. of input, output & $d,\ D$ & $C_{\cM}(d+1),\ D$\\
		\cline{2-4}
		&  $L_{\scrD}$ & $O\left(\log n\right)$ & $O\left(\log^2 n+\log D\right)$\\
		\cline{2-4}
		&  $p_{\scrD}$ & $O\left(n^{\frac{d}{d+2}}\right)$ & $O\left(Dn^{\frac{d}{d+2}}\right)$\\
		\cline{2-4}
		&  $K_{\scrD}$ & $O\left(n^{\frac{d}{d+2}}\log^2 n\right)$ & $O\left(Dn^{\frac{d}{d+2}} \log^2 n +D\log D\right)$\\
		\cline{2-4}
		&  $\kappa_{\scrD}$ & $O\left(n^{\frac{1}{d+2}}\right)$ & $O\left(n^{\frac{1}{d+2}}\right)$\\
		\cline{2-4}
		&  $R_{\scrD}$ & $B$ & $B$\\
		\hline
	\end{tabular}
	\caption{Comparison of the network architectures in Theorem \ref{thm.single}, \ref{thm.multi} and \ref{thm.gaussian}.}
	\label{tab.architecture}
\end{table}

\section{Numerical experiments}
\label{sec.experiments}

\begin{figure}[t!]
	\centering
	\subfigure[Sphere]{
	\includegraphics[scale=0.67]{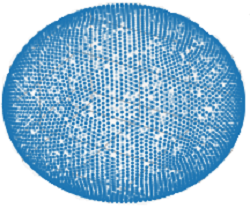}
	} \qquad
	\subfigure[Noisy data]{
    \includegraphics[scale=0.67]{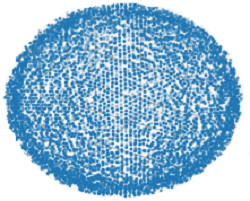}
    }\qquad
    \subfigure[Reconstruction]
    {\includegraphics[scale=0.67]{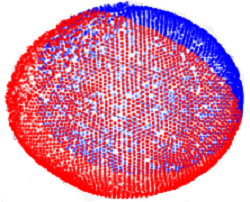}
    } \\ \vspace{10mm}
    \subfigure[Genus-3 pyramid]{
    \includegraphics[scale=0.67]{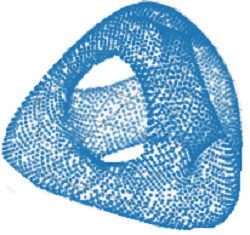}}\qquad
    \subfigure[Noisy data]{
    \includegraphics[scale=0.67]{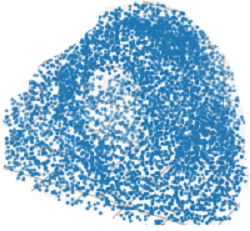}}\qquad
    \subfigure[Reconstruction]{
    \includegraphics[scale=0.67]{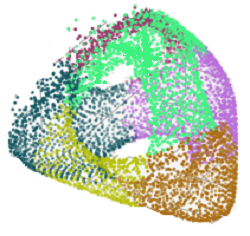}
    }\\ \vspace{10mm}
    \subfigure[Genus-2]{
    \includegraphics[scale=0.76]{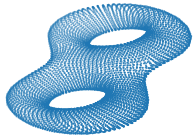}} \quad
    \subfigure[Noisy data]{
    \includegraphics[scale=0.76]{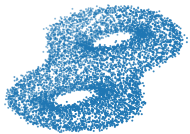}}\quad
    \subfigure[Reconstruction]{
    \includegraphics[scale=0.63]{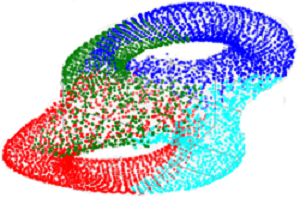}}
	\caption{Reconstructions of three noisy manifolds rescaled to lie inside the $[-50,50]^3$ cube centered at the origin.  } 
	\label{fig:visuals}
\end{figure}

In this section, we conduct a series of experiments on simulated data to numerically verify our theoretically analysis. We consider three surfaces listed in the first column of Figure \ref{fig:visuals}. The noisy data with normal noise are displayed in the middle column of Figure \ref{fig:visuals}. We use the code in \cite{schonsheck2019chart} to implement chart autoencoders. 
It is important to prescribe a reasonable number of charts to appropriately reflect the topology of the manifold. During training, chart autoencoders segment manifolds into largely non-overlapping charts. Further in some cases, chart autoencoders perform automatic chart pruning if too many charts are prescribed. This is done by contracting excess charts to a trivial or nearly trivial patch when there is already a sufficient number of charts present to capture the manifold's topology. The number of charts is prescribed to be $4$ for the sphere and Genus-2 double torus and $8$ for the Genus-3 pyramid. In Figure \ref{fig:visuals}, we visualize the sphere as decomposed by $2$ charts, a Genus-3 pyramid as decomposed by $6$ charts, and a Genus-2 double torus as decomposed by $4$ charts after training.
We use the network architecture such that the encoder is composed of 3 linear layers with ReLU activations, and the decoder has 3 linear layers with ReLU activations and a width (hidden dimension) of $50$.  In training, the batch size is $512$, the learning rate is $3e-6$, and weight decay is $3e-1$.

\subsection{Sample complexity}

\begin{figure}[t!]
    \centering
    \includegraphics[scale=0.6]{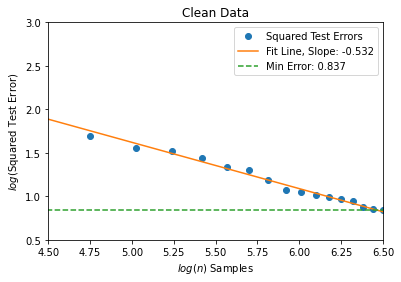}
    \caption{Squared test error versus the training sample size $n$ for clean data on the Genus-3 pyramid with intrinsic dimension $d=2$ and ambient dimension $D=3$. The horizontal line of \lq \lq Min Error\rq\rq\ represents the minimum squared testing error observed among all sample sizes. }
    \label{fig:sample_clean_3l}
\end{figure}

We first investigate the sample complexity of chart autoencoders.
We train chart autoencoders with $n$ training points randomly sampled from the Genus-3 pyramid and evaluate the squared generalization error on held-out test data. The training data contain clean and noisy pairs and the test data are noisy. Let $\{(\xb_j,\vb_j)\}_{j=1}^{n_{\rm test}}$ be the set of test data where $\vb_j$ is the clean counterpart of $\xb_j$. The squared generalization error is approximated by the following squared test error:
$$\frac{1}{n_{\rm test}} \sum_{j=1}^{n_{\rm test}} \| \widehat{\scrD}\circ\widehat{\scrE}(\xb_j) -\vb_j\|_2^2,$$
For each $n$, we perform $5$ runs of experiments and average the squared test error over these $5$ runs of experiments. 

According to our Theorem \ref{thm.multi}, if the training data contain clean and noisy pairs with normal noise, including the noise-free case, we have
\begin{equation}
\text{Squared generalization error} \le C  (D^2 \log^3 D) n^{-\frac{2}{d+2}\log^4 n}.
\label{num:errornormalnoise}
\end{equation}
If the training data contain clean and noisy pairs with high-dimensional noise where the second moment of the tangential component is bounded by $\sigma^2$, Theorem \ref{thm.gaussian} implies that
\begin{equation}
\text{Squared generalization error} \le C_1 (D^2 \log^3 D) n^{-\frac{2}{d+2}\log^4 n}+C_2\sigma^2.
\label{num:errorgaussiannoise}
\end{equation}

We start with the noise-free case where both training and test data are on the Genus-3 pyramid. Figure \ref{fig:sample_clean_3l} shows the log-log plot of the squared test error versus the training sample size $n$.  Our theory in \eqref{num:errornormalnoise} implies that a least square fit of the curve has slope $-\frac{2}{d+2} = -\frac 1 2$ since $d=2$. Numerically we obtain a slope of $-0.532$, which is consistent with our theory. 
Due to the optimization error  in training, we do not observe convergence to $0$ in either the training or the test loss. The \lq\lq Min error\rq\rq\ is the minimum squared test error achieved among all sample sizes.

\begin{figure}[t!]
	\centering
	\subfigure[Normal noise, $D=3$]{
	\includegraphics[scale=0.5]{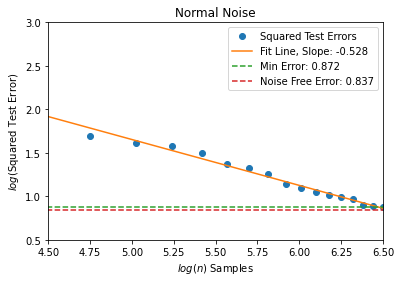} }
	\subfigure[Gaussian noise, $D=3$]{
    \includegraphics[scale=0.5]{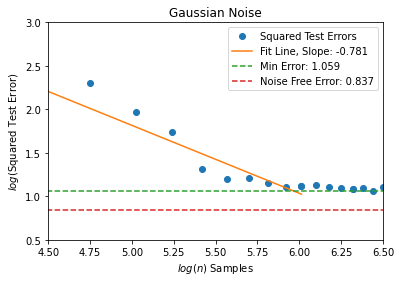}} \\
    \subfigure[Normal noise, $D=5$]{
    \includegraphics[scale=0.5]{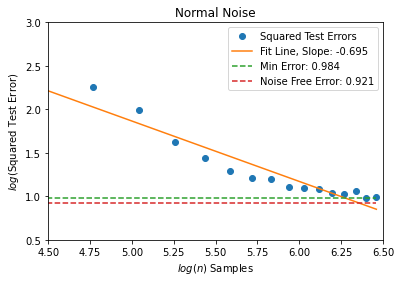} }
    \subfigure[Gaussian noise, $D=5$]{
    \includegraphics[scale=0.5]{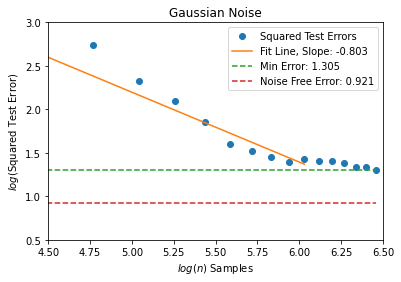}}\\
    \subfigure[Normal noise, $D=10$]{
    \includegraphics[scale=0.5]{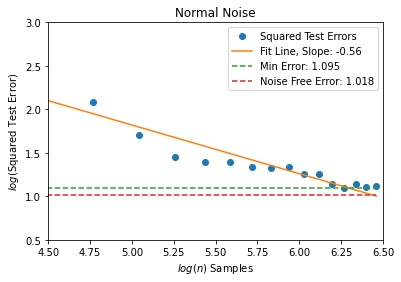} }
    \subfigure[Gaussian noise, $D=10$]{
    \includegraphics[scale=0.5]{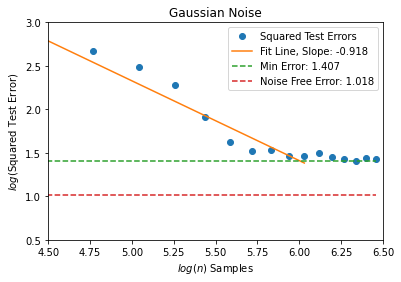}}
	\caption{Squared test error versus the training sample size $n$ on the Genus-3 pyramid with intrinsic dimension $d=2$ and ambient dimension $D \in \{3,5,10\}$. The left column shows the results of normal noise and the right column shows the results of Gaussian noise. 
 The \lq\lq Min error\rq\rq\ is the minimum squared test error achieved among all sample sizes. The \lq\lq Noise Free Error\rq\rq\ is the squared test error achieved when training on the entire clean dataset.}
	\label{fig:samples}
\end{figure}

We next test the noisy case with normal noise and gaussian noise respectively.
Figure \ref{fig:samples} displays the log-log plot of the squared test error versus the training sample size $n$ with normal noise (left column) and gaussian noise (right column). The noise level is measured by the variance of the noise distribution. Specifically,  for the normal noise $\wb$ in Setting \ref{setting}, we set $\EE[\wb|\vb]$ = 0 and refer the noise level as $\EE \|\wb\|^2$. For the gaussian noise $\nbb$, we set $\nbb \sim \mathcal{N}(0,\tilde\sigma^2 I)$, such that $\EE \|\nbb\|^2 =\tilde \sigma^2$, which is referred to be the noise level. In this experiment, we set the noise level to be $1$. 
The Genus-3 surface is embedded in $\RR^{D}$ with $D=3,5,10$ respectively. In Figure \ref{fig:samples}, the \lq\lq Min error\rq\rq\ is the minimum squared test error achieved among all sample sizes. The \lq\lq Noise Free Error\rq\rq\ is the squared test error achieved when training on the entire clean dataset.

 In the case of normal noise, we observe a convergence of the squared test error as $n$ increases. The  \lq\lq Min error\rq\rq\ is close to the \lq\lq Noise Free Error\rq\rq,\  which shows that training on noisy data almost achieves the performance of training on clean data. This demonstrates autoencoders' denoising effect for normal noise. The slope of the line obtained from a linear fit is around $-0.5$, which is consistent with our theory in \eqref{num:errornormalnoise}.

In the case of gaussian noise, the squared test error first converges when $n$ increases but then stagnates at a certain level. The  \lq\lq Min error\rq\rq\ is much larger than the \lq\lq Noise Free Error\rq\rq\, which shows training on noisy data can not achieve similar performance on clean data. This is expected as our theory implies that autoencoders do not have a denoising effect for the tangential component of the noise.

\subsection{Effects of the ambient dimension, the number of charts and noise levels}

We next investigate how the squared generalization error of chart autoencoders depends on the ambient dimension, the number of charts, and noise levels.

Our theory in \eqref{num:errornormalnoise} shows that, when the ambient dimension $D$ varies, the squared generalization error grows at most in $D^2 \log^3 D$. This bound may not be tight on the dependence of $D$. In Figure \ref{fig:effects} (a), we plot the squared test error for chart autoencoders with clean data on the Genus-2 and Genus-3 surfaces.  We observe that, in these simulations, the squared test error almost grows linearly with respect to $D$. 
Note that the upper bounds in our theorems are for the global minimizer of the empirical loss (\ref{eq.loss}). In practice, due to the complicated structure of networks, the training process may easily get stuck at a local minimizer and it is difficult to get the global minimizer. Nevertheless, our numerical results still give an approximate linear relation between the test error and $D$.
We leave it as a future work to investigate the optimal dependence of the squared generalization error on $D$.

\begin{figure}[t!]
	\centering
	\subfigure[Squared test error versus $D$]{
	\includegraphics[scale=0.5]{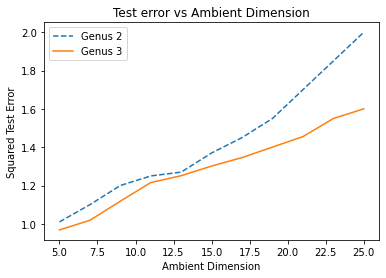}
	}
	\subfigure[Squared test error versus the number of charts]{	\includegraphics[scale=0.5]{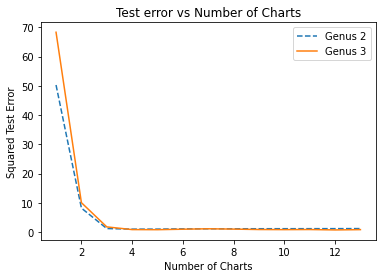}
}
\\
\subfigure[Squared test error versus noise levels]{	\includegraphics[scale=0.5]{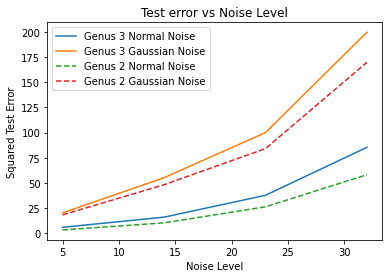}
}
	\caption{Effects of the ambient dimension, the number of charts and noise levels for the Genus-3 and Genus-2 data.}
	\label{fig:effects}
\end{figure}

In Figure \ref{fig:effects} (b), we plot the squared test error for chart autoencoders with clean data on the Genus-2 and Genus-3 surfaces versus  the number of charts.  We observe that, when the number of charts is sufficiently large to preserve the data structure, the squared test error stays almost the same, independently of the number of charts.

In Figure \ref{fig:effects} (c), we plot the squared test error versus various noise levels of both normal and gaussian noise. We observe that the squared test error is much higher in the gaussian case for both manifolds.

\section{Proof of main results}
\label{sec.proofofmain}
For the simplicity of notations, for any given $\cF_{\rm NN}^{\scrE}$ and $\cF_{\rm NN}^{\scrD}$, we define the network class
\begin{align}
	\cF_{\rm NN}^{\scrG}=\{\scrG=\scrD\circ \scrE ~|~ \scrD\in \cF_{\rm NN}^{\scrD},\scrE\in \cF_{\rm NN}^{\scrE}\}
	\label{eq.FG}
\end{align}
and denote $\widehat{\scrG}=\widehat{\scrD}\circ \widehat{\scrE}$, where $\widehat{\scrE},\widehat{\scrD}$ are the global minimizers in (\ref{eq.loss}).
\subsection{Proof of Theorem \ref{thm.single}}
\label{proof.single}
\begin{proof}[Proof of Theorem \ref{thm.single}]
	To simplify the notation, we denote the Lipschitz constant of $\fb$ by $C_f$, i.e.,
	for any $\vb_1,\vb_2\in\cM$,
	\begin{align}
		&\|\fb(\vb_1)-\fb(\vb_2)\|_{2}\leq C_f \|\vb_1-\vb_2\|_{2}\leq C_f d_{\cM}(\vb_1,\vb_2),
	\end{align}
    with $d_{\cM}(\vb_1,\vb_2)$ denoting the geodesic distance on $\cM$ between $\vb_1$ and $\vb_2$, and denote the Lipschitz constant of $\gb$ by $C_g$, i.e., for any $\zb_1,\zb_2\in \fb(\cM)$,
	\begin{align}
		&\|\gb(\zb_1)-\gb(\zb_2)\|_{2}\leq C_g \|\zb_1-\zb_2\|_2.
	\end{align}

Our proof idea can be summarized as: We decompose the generalization error (\ref{eq.error}) into a bias term and a variance term. The bias term will be upper bounded using network approximation theory in Lemma \ref{lem.single.approx}. The variance will be bounded in terms of the covering number of the network class  using metric entropy argument in Lemma \ref{lem.single.T2}. 

We add and subtract two times of the empirical risk to (\ref{eq.error}) to get
\begin{align}
	&\EE_{\cS}\EE_{\xb\sim \gamma} \left[\| \widehat{\scrG}(\xb)-\pi(\xb)\|_2^2 \right]\nonumber\\
	=& \underbrace{2\EE_{\cS}\left[\frac{1}{n}\sum_{i=1}^n \|\widehat{\scrG}(\xb_i) -\pi(\xb_i)\|_{2}^2\right]}_{\rm T_1}  +\underbrace{\EE_{\cS}\EE_{\xb\sim \gamma} \left[ \| \widehat{\scrG}(\xb) - \pi(\xb)\|_2^2\right] -2\EE_{\cS}\left[\frac{1}{n}\sum_{i=1}^n \|\widehat{\scrG}(\xb_i) - \pi(\xb_i)\|_{2}^2\right]}_{\rm T_2}.
	\label{eq.single.error}
\end{align}
The term $\mathrm {T_1}$ captures the bias of the network class $\cF_{\rm NN}^{\scrG}$ and ${\rm T_2}$ captures the variance.
We then derive an upper bound for each term in order.

\noindent $\bullet$ {\bf Bounding ${\rm T_1}$}

We derive an upper bound of ${\rm T_1}$ using the network approximation error. 
We deduce
\begin{align}
	{\rm T_1}=&2\EE_{\cS}\left[\frac{1}{n}\sum_{i=1}^n \|\widehat{\scrG}(\xb_i)- \pi(\xb_i)\|_{2}^2\right] \nonumber\\
	=&2\EE_{\cS}\inf_{\scrG\in \cF_{\rm NN}^{\scrG}}\left[\frac{1}{n}\sum_{i=1}^n \|\scrG(\xb_i)- \vb_i\|_{2}^2\right] \nonumber\\
	\leq & 2\inf_{\scrG\in \cF_{\rm NN}^{\scrG}} \EE_{\cS}\left[\frac{1}{n}\sum_{i=1}^n \|\scrG(\xb_i)- \vb_i\|_{2}^2\right] \nonumber\\
	=&2\inf_{\scrG\in \cF_{\rm NN}^{\scrG}} \EE_{\xb\sim \gamma}\left[\|\scrG(\xb)- \pi(\xb)\|_{2}^2\right].
	\label{eq.single.T1.err.1}
\end{align}

The following Lemma \ref{lem.single.approx} shows that by properly choosing the architecture of $\cF_{\rm NN}^{\scrE}$ and $\cF_{\rm NN}^{\scrD}$, there exist $\widetilde{\scrE}\in \cF_{\rm NN}^{\scrE}$ and $\widetilde{\scrD}\in\cF_{\rm NN}^{\scrD}$ so that that $\widetilde{\scrD}\circ\widetilde{\scrE}(\xb)$ approximates $\pi(\xb)$ with high accuracy:
\begin{lemma}\label{lem.single.approx}
	Consider Setting \ref{setting} and suppose Assumption \ref{assum.single} holds. For any $0<\varepsilon<\min\{1,\tau/2\}$, there exist two network architectures $\cF_{\rm NN}^{\scrE}=\cF(D,d;L_{\scrE},p_{\scrE},K_{\scrE},\kappa_{\scrE},R_{\scrE})$ and $\cF_{\rm NN}^{\scrD}=\cF(d,D;L_{\scrD},p_{\scrD},K_{\scrD},\kappa_{\scrD},R_{\scrD})$ with
	\begin{align}
		L_{\scrE}=O\left(\log^2 \varepsilon^{-1}+\log D\right), \ p_{\scrE}=O\left(D\varepsilon^{-d}\right), \ K_{\scrE}=O\left(D\varepsilon^{-d}\log^2 \varepsilon+D\log D\right),\ \kappa_{\scrE}=O\left(\varepsilon^{-2}\right), \ R_{\scrE}=\Lambda
		\label{eq.E.para.1}
	\end{align}
and 
\begin{align}
	L_{\scrD}=O\left(\log \varepsilon^{-1}\right), \ p_{\scrD}=O\left(D\varepsilon^{-d}\right), \ K_{\scrD}=O\left(D\varepsilon^{-d}\log \varepsilon\right),\ \kappa_{\scrD}=O\left(\varepsilon^{-1}\right), R_{\scrD}=B,
	\label{eq.D.para.1}
\end{align}
where the constant hidden in $O$ depends on $d,B,\Lambda,\tau,q, C_f,C_g$ and the volume of $\cM$.
These network architectures give rise to $\widetilde{\scrE}: \cM(q)\rightarrow \RR^d$ in $\cF_{\rm NN}^{\scrE}$ and $\widetilde{\scrD}: \RR^d\rightarrow \RR^D$ in $\cF_{\rm NN}^{\scrD}$ so that
\begin{align}
	\sup_{\xb\in\cM(q)}\|\widetilde{\scrD}\circ\widetilde{\scrE}(\xb)-\pi(\xb)\|_{\infty}\leq \varepsilon.
\end{align}

\end{lemma}

Lemma \ref{lem.single.approx} is proved in Appendix \ref{proof.single.approx}. The proof is based on the approximation theory in \cite{cloninger2021deep}.

Let $\widetilde{\scrE}$ and $\widetilde{\scrD}$ be the networks in Lemma \ref{lem.single.approx} with accuracy $\varepsilon$. We have 
\begin{align}
	{\rm T_1}\leq & 2\inf_{\scrG\in \cF_{\rm NN}^\scrG} \EE_{\xb\sim \gamma}\left[\|\scrG(\xb)- \pi(\xb)\|_{2}^2\right] \nonumber\\
	\leq & 2\EE_{\xb\sim \gamma}\left[\|\widetilde{\scrD}\circ\widetilde{\scrE}(\xb)- \pi(\xb)\|_{2}^2\right] \nonumber\\
	\leq & 2D\sup_{\xb\in\cM(q)}\|\widetilde{\scrD}\circ\widetilde{\scrE}(\xb)- \pi(\xb)\|_{\infty}^2 \nonumber\\
	\leq &2D\varepsilon^2. 
	\label{eq.single.T1}
\end{align}

\noindent $\bullet$ {\bf Bounding ${\rm T_2}$.} The term ${\rm T_2} $ is the difference between the population risk and empirical risk of the network class $\cF_{\rm NN}^{\scrG}$, except the empirical risk has a factor 2. We will derive an upper bound of ${\rm T_2}$ using the covering number of $\cF_{\rm NN}^{\scrG}$. The cover and covering number of a function class are defined as
\begin{definition}[Cover]\label{def.cover}
	Let $\cF$ be a class of functions. A set of functions $\cS$ is a $\delta$-cover of $\cF$ with respect to a norm $\|\cdot \|$ if for any $f\in \cF$, one has
	$$
	\inf_{u\in \cS}\|u - f\|\leq \delta.
	$$
\end{definition}
\begin{definition}[Covering number, Definition 2.1.5 of \citep{wellner2013weak}]\label{def.covernumber}
	Let $\cF$ be a class of functions. For any $\delta>0$, the covering number of $\cF$ is defined as
	$$
	\cN(\delta,\cF,\|\cdot\|)=\min\{ |\cS_f| : \text{$\cS_f$ is a $\delta$-cover of $\cF$ under $\| \cdot \|$ } \},
	$$
	where $|\cS_f|$ denotes the cardinality of $\cS_f$.
\end{definition}

The following lemma gives an upper bound of ${\rm T_2}$:
\begin{lemma}\label{lem.single.T2}
	Consider Setting \ref{setting}. Let $0<\delta<1$ and ${\rm T_2}$ be defined as in (\ref{eq.single.error}). We have 
	\begin{align}
		{\rm T_2} \leq \frac{35DB^2}{n}\log \cN\left( \frac{\delta}{2DB},\cF_{\rm NN}^{\scrG},\|\cdot\|_{L^{\infty,\infty}} \right)+6\delta.
		\label{eq.single.T2}
	\end{align}
\end{lemma}
Lemma \ref{lem.single.T2} is proved in Appendix \ref{proof.single.T2}

\noindent $\bullet$ {\bf Putting both gradients together.}
Putting (\ref{eq.single.T1}) and (\ref{eq.single.T2}) together, we have 
\begin{align}
	\EE_{\cS}\EE_{\xb\sim \gamma} \left[\| \widehat{\scrG}(\xb)-\pi(\xb)\|_2^2\right]\leq& 2D\varepsilon^2+\frac{35DB^2}{n}\log \cN\left( \frac{\delta}{2DB},\cF_{\rm NN}^{\scrG},\|\cdot\|_{L^{\infty,\infty}} \right)+6\delta.
	\label{eq.err.1}
\end{align}

 The following lemma gives an upper bound of $\cN\left( \frac{\delta}{2DB},\cF_{\rm NN}^{\scrG},\|\cdot\|_{\infty,\infty} \right)$ in terms of the network architecture (see a proof in Appendix \ref{proof.FG.covering}):
\begin{lemma}\label{lem.FG.covering}
	Let $\cF_{\rm NN}^{\scrG}$ be defined in (\ref{eq.FG}). 
	The covering number of $\cF_{\rm NN}^{\scrG}$ is bounded by
	\begin{align}
		&\cN\left( \delta,\cF_{\rm NN}^{\scrG},\|\cdot\|_{L^{\infty,\infty}} \right) \leq \nonumber\\
		&\left( \frac{2(L_{\scrE}+L_{\scrD})^2(\max\{p_{\scrE},p_{\scrD}\}B+2)(\max\{\kappa_{\scrE},\kappa_{\scrD}\})^{L_{\scrE}+L_{\scrD}} (\max\{p_{\scrE},p_{\scrD}\})^{L_{\scrE}+L_{\scrD}+1}}{\delta}\right)^{2(K_{\scrE}+K_{\scrD})}.
	\end{align}
\end{lemma}

Substituting (\ref{eq.E.para.1}) and (\ref{eq.D.para.1}) into Lemma \ref{lem.FG.covering}, we have 
\begin{align}
	&\log \cN\left( \frac{\delta}{2DB},\cF_{\rm NN}^{\scrG},\|\cdot\|_{L^{\infty,\infty}} \right) \nonumber\\
	= &O\left((L_{\scrE}+L_{\scrD})(K_{\scrE}+K_{\scrD})\left(\log (L_{\scrE}+L_{\scrD})+ \log (\max\{p_{\scrE},p_{\scrD}\}) + \log (\max\{\kappa_{\scrE},\kappa_{\scrD}\}) +\log \delta^{-1}\right)\right) \nonumber\\
	= & O\left((\log^2 \varepsilon^{-1}+\log D) (D\varepsilon^{-d}\log\varepsilon^{-1}+D\log D)\left( \log \varepsilon^{-1}+\log D +\log \delta^{-1}\right)\right) \nonumber\\
	\leq & O\left((D\log^2D)\varepsilon^{-d}\log^4\varepsilon^{-1} + (D\log^2D)\varepsilon^{-d}\log^3\varepsilon^{-1}\log \delta^{-1}\right).
	\label{eq.logCover}
\end{align}
Substituting (\ref{eq.logCover}) into (\ref{eq.err.1}) and setting $\delta=\varepsilon=n^{-\frac{1}{d+2}}$ gives rise to
\begin{align}
	\EE_{\cS}\EE_{\xb\sim \gamma} \left[\| \widehat{\scrG}(\xb)-\pi(\xb)\|_2^2\right]\leq CD^2\left(\log^2 D\right)n^{-\frac{2}{d+2}}\log^4 n
\end{align}
for some constant $C$ depending on $d,C_f,C_g,B,R_1,\tau$ and the volume of $\cM$.
The network sizes for $\cF_{\rm NN}^E$ and $\cF_{\rm NN}^D$ are given as
\begin{align*}
	&L_{\scrE}=O\left(\log^2 n+\log D\right), \ p_{\scrE}=O\left(D n^{\frac{d}{d+2}}\right), \ K_{\scrE}=O\left(Dn^{\frac{d}{d+2}}\log^2 n+ D\log D\right),\ \kappa_{\scrE}=O\left(n^{\frac{2}{d+2}}\right),\\
	&L_{\scrD}=O\left(\log n\right), \ p_{\scrD}=O\left(n^{\frac{d}{d+2}}\right), \ K_{\scrD}=O\left(n^{\frac{d}{d+2}}\log^2 n\right),\ \kappa_{\scrD}=O\left(n^{\frac{1}{d+2}}\right).
\end{align*}

\end{proof}

\subsection{Proof of Theorem \ref{thm.multi}}
\label{proof.multi}
\begin{proof}[Proof of Theorem \ref{thm.multi}]
	The proof of Lemma \ref{thm.multi} is similar to that of Lemma \ref{thm.single}, except extra efforts are needed to define the oracle encoder and decoder.
	
	Similar to the proof of Theorem \ref{thm.single}, we decompose the generalization error as
	\begin{align}
		&\EE_{\cS}\EE_{\xb\sim \gamma} \left[\| \widehat{\scrG}(\xb)-\pi(\xb)\|_2^2 \right]\nonumber\\
		=& \underbrace{2\EE_{\cS}\left[\frac{1}{n}\sum_{i=1}^n \|\widehat{\scrG}(\xb_i) -\pi(\xb_i)\|_{2}^2\right]}_{\rm T_1}  +\underbrace{\EE_{\cS}\EE_{\xb\sim \gamma} \left[ \| \widehat{\scrG}(\xb) - \pi(\xb)\|_2^2\right] -2\EE_{\cS} \left[\frac{1}{n}\sum_{i=1}^n \|\widehat{\scrG}(\xb_i) - \pi(\xb_i)\|_{2}^2 \right]}_{\rm T_2}.
		\label{eq.multi.decom}
	\end{align}

	\noindent $\bullet$ {\bf Bounding ${\rm T_1}$.}	
	We derive an upper bound of ${\rm T_1}$ using network approximation error. 
	
	Following (\ref{eq.single.T1.err.1}), we have
	\begin{align}
		{\rm T_1}\leq &2\inf_{\scrG\in \cF_{\rm NN}^{\scrG}} \EE_{\xb\sim \gamma}\left[\|\scrG(\xb)- \pi(\xb)\|_{2}^2\right].
	\end{align}
	The following Lemma shows that by properly choosing the architecture of $\cF_{\rm NN}^{\scrE}$ and $\cF_{\rm NN}^{\scrD}$, then there exist $\widetilde{\scrE}$ and $\widetilde{\scrD}$ that $\widetilde{\scrD}\circ\widetilde{\scrE}(\xb)$ approximates $\pi(\xb)$ with high accuracy :
	\begin{lemma}\label{lem.multi.approx}
		Consider Setting \ref{setting}. For any $0<\varepsilon<\min\{1,\tau/2\}$, there exist two network architectures  $\cF_{\rm NN}^{\scrE}(D,C_{\cM}(d+1);L_{\scrE}, p_{\scrE}, K_{\scrE},\kappa_{\scrE},R_{\scrE})$ and $\cF_{\rm NN}^{\scrD}(C_{\cM}(d+1),D;L_{\scrD}, p_{\scrD}, K_{\scrD},\kappa_{\scrD},R_{\scrD})$ with
		\begin{align}
			&L_{\scrE}=O(\log^2 \varepsilon^{-1}+\log D), \ p_{\scrE}=O(D\varepsilon^{-d}), \ K_{\scrE}=O((D\log D)\varepsilon^{-d}\log^2 \varepsilon^{-1}), \nonumber\\
			& \kappa_{\scrE}=O(\varepsilon^{-2}), \ R_{\scrE}=\max\{\tau/4,1\}.
			\label{eq.multi.E.para.1}
		\end{align} 
		and
		\begin{align}
			&L_{\scrD}=O(\log^2 \varepsilon^{-1}+\log D), \ p_{\scrD}=O(D\varepsilon^{-d}), \ K_{\scrD}=O(D\varepsilon^{-d} \log^2 \varepsilon +D\log D), \nonumber\\
			& \kappa_{\scrD}=O(\varepsilon^{-1}), \ R_{\scrD}=B.
			\label{eq.multi.D.para.1}
		\end{align}
		The constant hidden in $O$ depends on $d,\tau,q,B,C_{\cM}$ and the volume of $\cM$.
		These network architectures give rise to $\widetilde{\scrE}:\cM(q)\rightarrow \RR^d$ in $\cF_{\rm NN}^{\scrE}$ and $\widetilde{\scrD}: \RR^d\rightarrow \RR^D$ in $\cF_{\rm NN}^{\scrD}$ so that
		\begin{align}
			\sup_{\xb\in \cM(q)}\|\widetilde{\scrD}\circ\widetilde{\scrE}(\xb)-\pi(\xb)\|_{\infty}\leq \varepsilon.
		\end{align}
	\end{lemma}
	Lemma \ref{lem.multi.approx} is proved by carefully designing an oracle encoder and decoder and showing that they can be approximated well be neural networks. The proof of Lemma \ref{lem.multi.approx} is presented in Appendix \ref{proof.multi.approx}.
	
	Let $\widetilde{\scrE}$ and $\widetilde{\scrD}$ be the networks in Lemma \ref{lem.multi.approx} so that 
	\begin{align}
		\sup_{\xb\in \cM(q)}\|\widetilde{\scrD}\circ\widetilde{\scrE}(\xb)-\pi(\xb)\|_{\infty}\leq \varepsilon.
	\end{align}
We can bound ${\rm T_1}$ as
\begin{align}
	{\rm T_1}\leq & 2\inf_{\scrG\in \cF_{\rm NN}^\scrG} \EE_{\xb\sim \gamma}\left[\|\scrG(\xb)- \pi(\xb)\|_{2}^2\right] \nonumber\\
	\leq & 2\EE_{\xb\sim \gamma}\left[\|\widetilde{\scrD}\circ\widetilde{\scrE}(\xb)- \pi(\xb)\|_{2}^2\right] \nonumber\\
	\leq & 2D\sup_{\xb\in \cM(q)}\|\widetilde{\scrD}\circ\widetilde{\scrE}(\xb)- \pi(\xb)\|_{\infty}^2 \nonumber\\
	\leq &2D\varepsilon^2.
	\label{eq.multi.T1}
\end{align}
	
\noindent $\bullet$ {\bf Bounding ${\rm T_2}$.}
By Lemma \ref{lem.single.T2}, we have
\begin{align}
	{\rm T_2} \leq \frac{35DB^2}{n}\log \cN\left( \frac{\delta}{2DB},\cF_{\rm NN}^{\scrG},\|\cdot\|_{L^{\infty,\infty}} \right)+6\delta.
	\label{eq.multi.T2}
\end{align}

\noindent $\bullet$ {\bf Putting both ingredients together.}
Combining (\ref{eq.multi.T1}) and (\ref{eq.multi.T2}) gives rise to
\begin{align}
	\EE_{\cS}\EE_{\xb\sim \gamma} \left[ \| \widehat{\scrG}(\xb)-\pi(\xb)\|_2^2 \right] \leq 2D\varepsilon^2 +\frac{35DB^2}{n}\log \cN\left( \frac{\delta}{2DB},\cF_{\rm NN}^G,\|\cdot\|_{L^{\infty,\infty}} \right)+6\delta.
	\label{eq.multi.err.1}	
\end{align}
The covering number can be bounded by substituting (\ref{eq.multi.E.para.1}) and (\ref{eq.multi.D.para.1}) into Lemma \ref{lem.FG.covering}:
\begin{align}
	&\log \cN\left( \frac{\delta}{2DB},\cF_{\rm NN}^{\scrG},\|\cdot\|_{L^{\infty,\infty}} \right) \nonumber\\
	= &O\left((L_{\scrE}+L_{\scrD})(K_{\scrE}+K_{\scrD})\left(\log (L_{\scrE}+L_{\scrD})+ \log (\max\{p_{\scrE},p_{\scrD}\}) + \log (\max\{\kappa_{\scrE},\kappa_{\scrD}\}) +\log \delta^{-1}\right)\right) \nonumber\\
	= & O\left((\log^2 \varepsilon^{-1}+\log D) ((D\log D)\varepsilon^{-d}\log\varepsilon^{-1})\left( \log \varepsilon^{-1}+\log D +\log \delta^{-1}\right)\right) \nonumber\\
	=& O\left( (D\log^3 D)\varepsilon^{-d}\log^4\varepsilon^{-1} + (D\log^3 D)\varepsilon^{-d}\log^3\varepsilon^{-1}\log \delta^{-1}\right).
	\label{eq.multi.logCover}
\end{align}
Substituting (\ref{eq.multi.logCover}) into (\ref{eq.multi.err.1}) and setting $\epsilon=\delta=n^{-\frac{1}{d+2}}$ give rise to
\begin{align}
	\EE_{\cS}\EE_{\xb\sim \gamma} \left[ \| \widehat{\scrG}(\xb)-\pi(\xb)\|_2^2 \right] \leq C(D^2\log^3D)n^{-\frac{2}{d+2}}\log^4n
	\label{eq.multi.err.2}	
\end{align}
for some constant $C$ depending on $d,\tau,q,B,C_{\cM}$ and the volume  of $\cM$.

Consequently, the network architecture  $\cF_{\rm NN}^{\scrE}(D,C_{\cM}(d+1);L_{\scrE}, p_{\scrE}, K_{\scrE},\kappa_{\scrE},R_{\scrE})$  has
\begin{align}
	&L_{\scrE}=O(\log^2 n+\log D), \ p_{\scrE}=O(Dn^{\frac{d}{d+2}}), \ K_{\scrE}=O((D\log D)n^{\frac{d}{d+2}}\log^2 n), \nonumber\\
	& \kappa_{\scrE}=O(n^{\frac{2}{d+2}}), \ R_{\scrE}=\max\{\tau/4,1\},
	\label{eq.multi.E.para.2}
\end{align} 
and the network architecture $\cF_{\rm NN}^{\scrD}(C_{\cM}(d+1),D;L_{\scrD}, p_{\scrD}, K_{\scrD},\kappa_{\scrD},R_{\scrD})$ has
\begin{align}
	&L_{\scrD}=O(\log^2 n+\log D), \ p_{\scrD}=O(Dn^{\frac{d}{d+2}}), \ K_{\scrD}=O(Dn^{\frac{d}{d+2}} \log^2 n +D\log D), \nonumber\\
	& \kappa_{\scrD}=O(n^{\frac{1}{d+2}}), \ R_{\scrD}=B.
	\label{eq.multi.D.para.2}
\end{align}
The constant hidden in $O$ depends on $d,\tau,q,B,C_{\cM}$ and the volume  of $\cM$.
\end{proof}

\subsection{Proof of Theorem \ref{thm.gaussian}}\label{proof.gaussian}
\begin{proof}[Proof of Theorem \ref{thm.gaussian}]
	Theorem \ref{thm.gaussian} can be proved by following the proof of Theorem \ref{thm.multi}.
	We decompose the generalization error as
	\begin{align}
		&\EE_{\cS}\EE_{\xb\sim \gamma} \left[\| \widehat{\scrG}(\xb)-\vb\|_2^2 \right] \nonumber\\
		=& \underbrace{2\EE_{\cS}\left[\frac{1}{n}\sum_{i=1}^n \|\widehat{\scrG}(\xb_i) -\vb_i\|_{2}^2\right]}_{\rm T_1}  +\underbrace{\EE_{\cS}\EE_{\xb\sim \gamma} \left[\| \widehat{\scrG}(\xb) - \pi(\xb)\|_2^2 \right]-2\EE_{\cS} \left[\frac{1}{n}\sum_{i=1}^n \|\widehat{\scrG}(\xb_i) - \vb_i\|_{2}^2\right] }_{\rm T_2}.
		\label{eq.gaussian.decom}
	\end{align}

	\noindent $\bullet$ {\bf Bounding ${\rm T_1}$.}	
	 Denote $\wb=\rm{Proj}_{T_\vb^{\perp}{\cM}}(\nbb)$ as the component of $\nbb$ that is normal to $\cM$ at $\vb$, and $\ub=\rm{Proj}_{T_{\vb}\cM}(\nbb)$ as the component that is in the tangent space of $\cM$ at $\vb$. Using Lemma \ref{lem.multi.approx} and the data model in Setting \ref{setting.gaussian}, we have 	
	 \begin{align}
		{\rm T_1}\leq & 2\inf_{\scrG\in \cF_{\rm NN}^\scrG} \EE_{\xb\sim \gamma}\left[\|\scrG(\xb)- \vb\|_{2}^2\right] \nonumber\\
		\leq & 2\EE_{\xb\sim \gamma}\left[2\|\widetilde{\scrD}\circ\widetilde{\scrE}(\xb)- \pi(\xb)\|_{2}^2 + 2\| \pi(\xb)-\vb\|_{2}^2\right] \nonumber\\
		\leq & 4D\|\widetilde{\scrD}\circ\widetilde{\scrE}(\xb)- \pi(\xb)\|_{\infty}^2 + 4\EE_{\xb\sim \gamma} \left[\| \pi(\xb)-\vb\|_{2}^2 \right] \nonumber\\
		\leq & 4D\varepsilon^2+ 4\EE_{\xb\sim \gamma} \left[\| \pi(\vb+\wb+\ub)-\pi(\vb+\wb)\|_{2}^2 \right]\nonumber\\
		\leq & 4D\varepsilon^2 +  4L_{\pi}\EE_{\xb\sim \gamma} \left[\|\ub\|_2^2\right] \nonumber\\
		= & 4D\varepsilon^2 +  4L^2_{\pi}\EE_{\vb}\EE_{\nbb} \left[\|\rm{Proj}_{T_{\vb}\cM}(\nbb)\|_2^2|\vb\right] \nonumber\\
		\leq & 4D\varepsilon^2+ 4L^2_{\pi}\sigma^2,
		\label{eq.gaussian.T1.1}
	\end{align}
	where $L_{\pi}$ denotes the Lipschitz constant of $\pi$. In (\ref{eq.gaussian.T1.1}), the fourth inequality uses Lemma \ref{lem.multi.approx}, the fifth inequality uses the fact that $\pi(\vb+\wb)=\pi(\vb)$.
	
	The following Lemma gives an upper bound of $L_{\pi}$:
	\begin{lemma}[Lemma 2.1 of \cite{cloninger2021deep}]\label{lem.pi}
		Let $\cM$ be a connected, compact, $d$-dimensional Riemannian manifold embedded in $\RR^D$ with a reach $\tau>0$. Let $\pi$ be the orthogonal projection onto $\cM$. For any $q\in[0,\tau)$, we have
		\begin{align}
			\|\pi(\xb_1)-\pi(\xb_2)\|_2\leq \frac{1}{1-q/\tau}\|\xb_1-\xb_2\|_2
		\end{align}
		for any $\xb_1,\xb_2\in \cM(q)$.
	\end{lemma}
	
	According to Lemma \ref{lem.pi}, $L_{\pi}$ only depends on $\tau$ and $q$. Therefore  we have
	\begin{align}
		{\rm T_1}\leq 4D\varepsilon^2+C_1\sigma^2
		\label{eq.gaussian.T1}
	\end{align}
	for some $C_1$ depending on $\tau$ and $q$.
	
	\noindent $\bullet$ {\bf Bounding ${\rm T_2}$.}
	By Lemma \ref{lem.single.T2}, we have
	\begin{align}
		{\rm T_2} \leq \frac{35DB^2}{n}\log \cN\left( \frac{\delta}{2DB},\cF_{\rm NN}^{\scrG},\|\cdot\|_{L^{\infty,\infty}} \right)+6\delta.
		\label{eq.gaussian.T2}
	\end{align}
	
	\noindent $\bullet$ {\bf Putting both ingredients together.}
	Combining (\ref{eq.gaussian.T1}) and (\ref{eq.gaussian.T2}) gives rise to
	\begin{align}
		\EE_{\cS}\EE_{\xb\sim \gamma} \left[ \| \widehat{\scrG}(\xb)-\pi(\xb)\|_2^2 \right]\leq 4D\varepsilon^2 +\frac{35DB^2}{n}\log \cN\left( \frac{\delta}{2DB},\cF_{\rm NN}^G,\|\cdot\|_{L^{\infty,\infty}} \right)+6\delta+C_1\sigma^2.
		\label{eq.gaussian.err.1}	
	\end{align}
	An upper bound of the covering number is given in (\ref{eq.multi.logCover}).
	Substituting (\ref{eq.multi.logCover}) into (\ref{eq.gaussian.err.1}) and setting $\epsilon=n^{-\frac{1}{d+2}}, \delta=\frac{1}{n}$ give rise to
	\begin{align}
		\EE_{\cS}\EE_{\xb\sim \gamma} \left[\| \widehat{\scrG}(\xb)-\pi(\xb)\|_2^2 \right]\leq C(D^2\log^3D)n^{-\frac{2}{d+2}}\log^4n +C_1\sigma^2
		\label{eq.gaussian.err.2}	
	\end{align}
	for some constant $C$ depending on $d,\tau,q,B,C_{\cM}$ and the volume  of $\cM$.
	
	Consequently, the network architecture  $\cF_{\rm NN}^{\scrE}(D,C_{\cM}(d+1);L_{\scrE}, p_{\scrE}, K_{\scrE},\kappa_{\scrE},R_{\scrE})$  has
	\begin{align}
		&L_{\scrE}=O(\log^2 n+\log D), \ p_{\scrE}=O(Dn^{\frac{d}{d+2}}), \ K_{\scrE}=O((D\log D)n^{\frac{d}{d+2}}\log^2 n), \nonumber\\
		& \kappa_{\scrE}=O(n^{\frac{2}{d+2}}), \ R_{\scrE}=\max\{\tau/4,1\},
		\label{eq.gaussian.E.para.2}
	\end{align} 
	and the network architecture $\cF_{\rm NN}^{\scrD}(C_{\cM}(d+1),D;L_{\scrD}, p_{\scrD}, K_{\scrD},\kappa_{\scrD},R_{\scrD})$ has
	\begin{align}
		&L_{\scrD}=O(\log^2 n+\log D), \ p_{\scrD}=O(Dn^{\frac{d}{d+2}}), \ K_{\scrD}=O(Dn^{\frac{d}{d+2}} \log^2 n +D\log D), \nonumber\\
		& \kappa_{\scrD}=O(n^{\frac{1}{d+2}}), \ R_{\scrD}=B.
		\label{eq.gaussian.D.para.2}
	\end{align}
	The constant hidden in $O$ depends on $d,\tau,q,B,C_{\cM}$ and the volume  of $\cM$.

\end{proof}

\section{Conclusion}
\label{sec.conclusion}
This paper studies the generalization error of Chart Auto-Encoders (CAE), when the noisy data are concentrated around a $d$-dimensional manifold $\cM$ embedded in $\RR^D$. We assume that the training data are well controlled such that both the noisy data and their clean counterparts are available. 
When the noise is along the normal directions of $\cM$, we prove that the squared generalization error converges to $0$ at a fast rate in the order of $n^{-\frac{2}{2+d}}\log^4 n$. 
When the noise contains both normal and tangential components, we prove that the squared generalization error converges to a value proportional to the second moment of the tangential noise.
Our results are supported by experimental validation.
Our findings provide evidence that deep neural networks are capable of extracting low-dimensional nonlinear latent features from data, contributing to the understanding of the success of autoencoders.

\bibliographystyle{ims}
\bibliography{ref}

\newpage
\appendix
\section*{Appendix}
\section{Proof of lemmas}
\subsection{Proof of Lemma \ref{lem.single.approx}}
\label{proof.single.approx}
\begin{proof}[Proof of Lemma \ref{lem.single.approx}]
	We show that there exist $\widetilde{\scrE}\in \cF_{\rm NN}^{\scrE}$ and $\widetilde{\scrD}\in \cF_{\rm NN}^{\scrD}$ that approximate $\fb\circ\pi$ and $\gb$ with the given accuracy $\varepsilon$.
	The following lemma shows the existence of a network architecture with which a network  $\widetilde{\scrE}$ approximates $\fb\circ\pi$ with high accuracy.
	\begin{lemma}\label{lem.f}
		Consider Setting \ref{setting} and suppose Assumption \ref{assum.single} holds. For any $0<\varepsilon<\tau/2$, there exists a network architecture $\widetilde{\fb} \in \cF(D,d;L,p,K,\kappa,R)$ so that
		\begin{align*}
			\sup_{\xb\in \cM(q)}\|\widetilde{\fb}(\xb)-\fb\circ\pi(\xb)\|_{\infty}\leq \varepsilon.
		\end{align*}
		Such a network architecture has
		\begin{align*}
			L=O\left(\log^2 \varepsilon^{-1}+\log D\right), \ p=O\left(D\varepsilon^{-d}\right), \ K=O\left(D\varepsilon^{-d}\log^2 \varepsilon+D\log D\right),\ \kappa=O\left(\varepsilon^{-2}\right), \ R=\Lambda,
		\end{align*}
		where the constant hidden in $O$ depends on $d,B,\Lambda,\tau,q, C_f$ and the volume  of $\cM$. 
	\end{lemma}
	Lemma \ref{lem.f} can be proved using \citet[Theorem 2.2]{cloninger2021deep}. One only needs to stack $d$ scalar-valued networks together.
	
	To construct a network to approximate $\gb$, first note that $\gb$ is defined on $\fb(\cM)\subset [-\Lambda,\Lambda]^d$. The following lemma shows that $\gb$ can be extended to $[-\Lambda,\Lambda]^d$ while keeping the same Lipschitz constant:
 \begin{lemma}[Kirszbraun theorem \citep{kirszbraun1934zusammenziehende}]
\label{lem:extension}
    If $E\subset \RR^d$, then any Lipschitz function $\gb: E\rightarrow \RR^D$
 can be extended to the whole $\RR^d$
 keeping the Lipschitz constant of the original function. 
\end{lemma}

By Lemma \ref{lem:extension}, we extend $\gb$ to $[-\Lambda,\Lambda]^d$ so that the extended function is Lipschitz continuous with Lipschitz constant $C_g$. When there is no ambiguity, we still use $\gb$ to denote the extended function.
	The following lemma shows the existence of a network architecture with which a network  $\widetilde{\scrD}$ approximates $\gb$ on $[-\Lambda,\Lambda]^d$ with high accuracy.
	\begin{lemma}\label{lem.g}
		Consider Setting \ref{setting} and suppose Assumption \ref{assum.single} holds. For any $0<\varepsilon<1$, there exists a network architecture $\widetilde{\gb} 
\in\cF(d,D;L,p,K,\kappa,R)$ so that
		\begin{align}
			\sup_{\zb\in [-\Lambda,\Lambda]^d}\|\widetilde{\gb}(\zb)-\gb(\zb)\|_{\infty}\leq \varepsilon.
		\end{align}
		Such a network architecture has
		\begin{align}
			L=O\left(\log \varepsilon^{-1}\right), \ p=O\left(D\varepsilon^{-d}\right), \ K=O\left(D\varepsilon^{-d}\log \varepsilon\right),\ \kappa=O\left(\varepsilon^{-1}\right), R=B,
		\end{align}
		where the constant hidden in $O$ depends on $d,C_g,B,\Lambda$.
	\end{lemma}
	Lemma \ref{lem.g} can be proved using \cite[Theorem 1]{yarotsky2017error}. One only needs to stack $D$ scalar-valued networks together.
	
	For a constant $\varepsilon_1\in (0,\min\{1,\tau/2\})$, we choose $\cF_{\rm NN}^{\scrE}=\cF(D,d;L_{\scrE}, p_{\scrE},K_{\scrE}, \kappa_{\scrE}, R_{\scrE})$ with
	\begin{align}
		L_{\scrE}=O(\log^2 \varepsilon_1^{-1})+\log D, \ p_{\scrE}=O(D\varepsilon_1^{-d}), \ K_{\scrE}=O(D\varepsilon_1^{-d}\log^2 \varepsilon_1+D\log D),\ \kappa_{\scrE}=\varepsilon_1^{-2},\ R_{\scrE}=\Lambda,
	\end{align}
	and 
	$\cF_{\rm NN}^{\scrD}=\cF(d,D;L_{\scrD}, p_{\scrD},K_{\scrD},\kappa_{\scrD},R_{\scrD})$ with
	\begin{align}
		L_{\scrD}=O(\log \varepsilon_1^{-1}), \ p_{\scrD}=O(D\varepsilon_1^{-d}), \ K_{\scrD}=O(D\varepsilon_1^{-d}\log \varepsilon_1),\ \kappa_{\scrD}=\varepsilon_1^{-1},\ R_{\scrD}=B.
	\end{align}
	
	According to Lemma \ref{lem.f}, there exists $\widetilde{\scrE}\in \cF_{\rm NN}^{\scrE}$ such that
	\begin{align}
		\sup_{\xb\in \cM(q)}\|\widetilde{\scrE}(\xb)-\fb\circ\pi(\xb)\|_{\infty}\leq \varepsilon_1.
		\label{eq.single.E.err}
	\end{align}

	According to Lemma \ref{lem.g}, there exists $\widetilde{\scrD}\in \cF_{\rm NN}^{\scrD}$ such that
	\begin{align}
		\sup_{\zb\in [-\Lambda,\Lambda]^d}\|\widetilde{\scrD}(\zb)-\gb(\zb)\|_{\infty}\leq \varepsilon_1.
		\label{eq.single.D.err}
	\end{align}

	Putting (\ref{eq.single.E.err}) and (\ref{eq.single.D.err}) together and setting $\varepsilon_1=\frac{\varepsilon}{1+C_g\sqrt{d}}$, give rise to
	\begin{align}
		\sup_{\xb\in \cM(q)}\|\widetilde{\scrD}\circ \widetilde{\scrE}(\xb)- \pi(\xb)\|_{\infty}  
  =& \sup_{\xb\in \cM(q)}\|\widetilde{\scrD}\circ \widetilde{\scrE}(\xb)- \bg\circ\fb\circ\pi(\xb)\|_{\infty} \nonumber\\
		\leq &\sup_{\xb\in \cM(q)}\left( \|\widetilde{\scrD}\circ \widetilde{\scrE}(\xb)- \bg\circ\widetilde{\scrE}(\xb)\|_{\infty} +  \|\bg\circ\widetilde{\scrE}(\xb)- \bg\circ\fb\circ\pi(\xb)\|_{\infty} \right)\nonumber\\
		\leq & \varepsilon_1+C_g \sup_{\xb\in \cM(q)} \|\widetilde{\scrE}(\xb)- \fb\circ\pi(\xb)\|_{2} \nonumber\\
		\leq& \varepsilon_1+ C_g\sqrt{d}\varepsilon_1 \nonumber\\
		=& \varepsilon. \nonumber
	\end{align}
	The lemma is proved.
\end{proof}

\subsection{Proof of Lemma \ref{lem.single.T2}}
\label{proof.single.T2}
\begin{proof}[Proof of Lemma \ref{lem.single.T2}]

	Denote 
	\begin{align*}
		\widehat{h}(\xb)=\|\widehat{\scrG}(\xb)-\pi(\xb)\|_2^2.
	\end{align*}
	We have $\widehat{h}(\xb)\leq 4DB^2$ for any $\xb\in\cM(q)$ due to the definition of $\widehat{\scrG}$. We bound ${\rm T_2}$ as
	\begin{align}
		{\rm T_2}=&\EE_{\cS} \left[\EE_{\xb}\left[\widehat{h}(\xb)\right]-\frac{2}{n} \sum_{i=1}^n \widehat{h}(\xb_i)\right] \nonumber\\
		=& 2\EE_{\cS} \left[\EE_{\xb}\left[\widehat{h}(\xb)\right]-\frac{1}{n} \sum_{i=1}^n \widehat{h}(\xb_i) -\frac{1}{2} \EE_{\xb} \left[ \widehat{h}(\xb)\right]\right] \nonumber\\
		\leq& 2\EE_{\cS} \left[\EE_{\xb}\left[\widehat{h}(\xb)\right]-\frac{1}{n} \sum_{i=1}^n \widehat{h}(\xb_i) -\frac{1}{8DB^2} \EE_{\xb} \left[\widehat{h}^2(\xb)\right]\right],
	\end{align}
	where in the last inequality we used $\widehat{h}(\xb)\leq 4DB^2$ and
	\begin{align*}
		\frac{1}{4DB^2}\EE_{\xb} \left[ \widehat{h}^2(\xb) \right] \leq \EE_{\xb} \left[\widehat{h}(\xb)\right].
	\end{align*}
	Let $\widetilde{\cS}=\{\widetilde{\xb}_i\}_{i=1}^n$ be a ghost sample set that is independent to $\cS$. Define the set
	\begin{align*}
		\cH=\{h: h(\xb)=\|\scrG(\xb)-\pi(\xb)\|_2^2 \mbox{ for } \scrG\in \cF_{\rm NN}^{\scrG}\}.
	\end{align*}
	We have
	\begin{align}
		{\rm T_2}\leq & 2\EE_{\cS} \sup_{h\in \cH}\left[\EE_{\xb}\left[h(\xb)\right]-\frac{1}{n} \sum_{i=1}^n h(\xb_i) -\frac{1}{8DB^2} \EE_{\xb} \left[h^2(\xb)\right]\right] \nonumber\\
		\leq &2\EE_{\cS,\widetilde{\cS}} \sup_{h\in \cH}\left[\frac{1}{n} \sum_{i=1}^n \left(h(\widetilde{\xb}_i)-h(\xb_i)\right) -\frac{1}{16DB^2} \EE_{\xb,\widetilde{\xb}} \left[h^2(\widetilde{\xb})+h^2(\xb)\right]\right]
		\label{eq.T2.ss}
	\end{align}
	
	Denote the $\delta$--covering number of $\cH$ by $\cN(\delta,\cH,\|\cdot \|_{\infty})$ and let $\cH^*=\{h^*_j\}_{j\in \cN(\delta,\cH,\|\cdot \|_{\infty})}$ be a $\delta$--cover of $\cH$, namely,  for any $h\in \cH$, there exists $h^*\in \cH^*$ so that $\|h-h^*\|_{\infty}\leq \delta$. 
Therefore, we have
	\begin{align}
		h(\widetilde{\xb})-h(\xb)= &h(\widetilde{\xb})-h^*(\widetilde{\xb})+h^*(\widetilde{\xb})-h^*(\xb)+h^*(\xb)-h(\xb) \nonumber\\
		\leq & h^*(\widetilde{\xb})-h^*(\xb) +2\delta.
		\label{eq.T2.h.minus}
	\end{align}
	and
	\begin{align}
		h^2(\widetilde{\xb})+h^2(\xb)=& \left[h^2(\widetilde{\xb})-(h^*)^2(\widetilde{\xb})\right] + \left[ (h^*)^2(\widetilde{\xb})+(h^*)^2(\xb)\right] -\left[(h^*)^2(\xb)-h^2(\xb)\right] \nonumber\\
		= & (h^*)^2(\widetilde{\xb})+(h^*)^2(\xb) + \left[h(\widetilde{\xb})-h^*(\widetilde{\xb})\right] \left[h(\widetilde{\xb})+h^*(\widetilde{\xb})\right] - \left[h^*(\xb)-h(\xb)\right] \left[h^*(\xb)+h(\xb)\right] \nonumber\\
		\geq & (h^*)^2(\widetilde{\xb})+(h^*)^2(\xb) - \left|h(\widetilde{\xb})-h^*(\widetilde{\xb})\right| \left|h(\widetilde{\xb})+h^*(\widetilde{\xb})\right| - \left|h^*(\xb)-h(\xb)\right| \left|h^*(\xb)+h(\xb)\right| \nonumber\\
		\geq & (h^*)^2(\widetilde{\xb})+(h^*)^2(\xb) - 16DB^2\delta
		\label{eq.T2.h.plus}
	\end{align}
	Substituting (\ref{eq.T2.h.minus}) and (\ref{eq.T2.h.plus}) into (\ref{eq.T2.ss}) gives rise to
	\begin{align}
		{\rm T_2}\leq &2\EE_{\cS,\widetilde{\cS}} \sup_{h^*\in \cH^*}\left[\frac{1}{n} \sum_{i=1}^n \left(h^*(\widetilde{\xb}_i)-h^*(\xb_i)\right) -\frac{1}{16DB^2} \EE_{\xb,\widetilde{\xb}} \left[(h^*)^2(\widetilde{\xb})+(h^*)^2(\xb)\right]\right]+6\delta \nonumber\\
		= & 2\EE_{\cS,\widetilde{\cS}} \max_{j}\left[\frac{1}{n} \sum_{i=1}^n \left(h_j^*(\widetilde{\xb}_i)-h_j^*(\xb_i)\right) -\frac{1}{16DB^2} \EE_{\xb,\widetilde{\xb}} \left[(h_j^*)^2(\widetilde{\xb})+(h_j^*)^2(\xb)\right]\right]+6\delta
	\end{align} 
	Denote $\eta_j(\widetilde{\xb}_i,\xb_i)=h^*_j(\widetilde{\xb}_i)-h^*_j(\xb_i)$. We can check that $\EE_{\cS,\widetilde{\cS}}[\eta_j(\widetilde{\xb}_i,\xb_i)]=0$ for any $j=1,...,\cN(\delta,\cH,\|\cdot \|_{\infty})$. We compute the variance of $\eta_j(\widetilde{\xb}_i,\xb_i)$ as
	\begin{align}
		\Var[\eta_j(\widetilde{\xb}_i,\xb_i)]=\EE_{\cS,\widetilde{\cS}}\left[\eta^2_j(\widetilde{\xb}_i,\xb_i)\right]=\EE_{\cS,\widetilde{\cS}}\left[\left(h^*_j(\widetilde{\xb}_i)-h^*_j(\xb_i)\right)^2\right]  \leq 2\EE_{\cS,\widetilde{\cS}}\left[\left(h^*_j\right)^2(\widetilde{\xb}_i)+\left(h^*_j\right)^2(\xb_i)\right]. 
	\end{align}
	We thus have
	\begin{align}
		{\rm T_2}\leq \widetilde{\rm T}_2+6\delta
	\end{align}
	with 
	\begin{align}
		\widetilde{\rm T}_2=2\EE_{\cS,\widetilde{\cS}} \max_{j}\left[\frac{1}{n} \sum_{i=1}^n \left(\eta_j(\widetilde{\xb}_i,\xb_i) -\frac{1}{32DB^2}  \Var[\eta_j(\widetilde{\xb}_i,\xb_i)]\right) \right].
	\end{align}
	We next derive the moment generating function of $\eta_j(\widetilde{\xb}_i,\xb_i)$. For $0<t<\frac{3}{4DB^2}$, we have
	\begin{align}
		\EE_{\cS,\widetilde{\cS}}\left[ \exp\left(t \eta_j(\widetilde{\xb}_i,\xb_i)\right)\right]= &\EE_{\cS,\widetilde{\cS}}\left[1+ t{ \eta_j}(\widetilde{\xb}_i,\xb_i) + \sum_{k=2}^{\infty} \frac{t^k\eta_j^k(\widetilde{\xb}_i,\xb_i)}{k!}\right] \nonumber\\
		\leq & \EE_{\cS,\widetilde{\cS}}\left[1+ t{ \eta_j}(\widetilde{\xb}_i,\xb_i) + \sum_{k=2}^{\infty} \frac{t^k\eta_j^2(\widetilde{\xb}_i,\xb_i)(4DB^2)^{k-2}}{2\times 3^{k-2}}\right] \nonumber\\
		=& \EE_{\cS,\widetilde{\cS}}\left[1+ t{ \eta_j}(\widetilde{\xb}_i,\xb_i) +\frac{t^2\eta_j^2(\widetilde{\xb}_i,\xb_i)}{2} \sum_{k=2}^{\infty} \frac{t^{k-2}(4DB^2)^{k-2}}{3^{k-2}}\right] \nonumber\\
		=&\EE_{\cS,\widetilde{\cS}}\left[1+ t{ \eta_j}(\widetilde{\xb}_i,\xb_i) +\frac{t^2\eta_j^2(\widetilde{\xb}_i,\xb_i)}{2} \frac{1}{1-4DB^2t/3}\right] \nonumber\\
		=&1+t^2\Var[\eta_j(\widetilde{\xb}_j,\xb_j)]\frac{1}{2-8DB^2t/3} \nonumber\\
		\leq & \exp\left( \Var[\eta_j(\widetilde{\xb}_i,\xb_i)]\frac{3t^2}{6-8DB^2t}\right),
		\label{eq.T2.eta.moment}
	\end{align}
	where in the first inequality we used $|\eta_j(\widetilde{\xb},\xb)|\leq 4DB^2$.
	
	 Set $t=\widetilde{t}/n$. For $0<t<\frac{3}{4DB^2}$, we have
	\begin{align}
		\exp\left(\frac{\widetilde{t}\widetilde{\rm T}_2}{2}\right)=& \exp\left(\widetilde{t}\EE_{\cS,\widetilde{\cS}} \max_{j}\left[\frac{1}{n} \sum_{i=1}^n \left(\eta_j(\widetilde{\xb}_i,\xb_i) -\frac{1}{32DB^2}  \Var[\eta_j(\widetilde{\xb}_i,\xb_i)]\right) \right] \right) \nonumber\\
		\leq & \EE_{\cS,\widetilde{\cS}}\left[\exp\left(\widetilde{t} \max_{j}\left[\frac{1}{n} \sum_{i=1}^n \left(\eta_j(\widetilde{\xb}_i,\xb_i) -\frac{1}{32DB^2}  \Var[\eta_j(\widetilde{\xb}_i,\xb_i)]\right) \right] \right)\right] \nonumber\\
		\leq & \EE_{\cS,\widetilde{\cS}}\left[\sum_j\exp\left(\frac{\widetilde{t}}{n} \sum_{i=1}^n \left(\eta_j(\widetilde{\xb}_i,\xb_i) -\frac{1}{32DB^2}  \Var[\eta_j(\widetilde{\xb}_i,\xb_i)]\right)  \right)\right] \nonumber\\
		= & \EE_{\cS,\widetilde{\cS}}\left[\sum_j\exp\left( t \sum_{i=1}^n \left(\eta_j(\widetilde{\xb}_i,\xb_i) -\frac{1}{32DB^2}  \Var[\eta_j(\widetilde{\xb}_i,\xb_i)]\right)  \right)\right] \nonumber\\
		\leq & \sum_j \exp\left( \sum_{i=1}^n \left(\Var[\eta_j(\widetilde{\xb}_i,\xb_i)]\frac{3 t^2}{6-8DB^2t} -\frac{1}{32DB^2}t  \Var[\eta_j(\widetilde{\xb}_i,\xb_i)]\right)  \right) \nonumber\\
		=& \sum_j \exp\left(t \sum_{i=1}^n \left(\Var[\eta_j(\widetilde{\xb}_i,\xb_i)]\left(\frac{3t}{6-8DB^2t} -\frac{1}{32DB^2} \right) \right)  \right),
		\label{eq.tildeT2}
	\end{align}
	where the first inequality follows from Jensen's inequality and the thrid inequality uses (\ref{eq.T2.eta.moment}). Setting
	\begin{align}
		\frac{3t}{6-8DB^2t} -\frac{1}{32DB^2}=0
	\end{align}
	gives $t=\frac{3}{52DB^2}<\frac{3}{4DB^2}$ and $\widetilde{t}=\frac{3n}{52DB^2}$. Substituting the value of $\widetilde{t}$ into (\ref{eq.tildeT2}) gives rise to
	\begin{align}
		\frac{\widetilde{t}\widetilde{\rm T}_2}{2}\leq \log \left(\sum_{j} \exp(0)\right)=\log \cN(\delta,\cH, \|\cdot\|_{\infty}),
	\end{align}
	implying that
	\begin{align}
		\widetilde{\rm T}_2\leq \frac{2}{\widetilde{t}} \log \cN(\delta,\cH, \|\cdot\|_{\infty})= \frac{104DB^2}{3n}\log \cN(\delta,\cH, \|\cdot\|_{\infty}).
	\end{align}
	and
	\begin{align}
		{\rm T_2}\leq \frac{104DB^2}{3n}\log \cN(\delta,\cH, \|\cdot\|_{\infty})+6\delta \leq \frac{35DB^2}{n}\log \cN(\delta,\cH, \|\cdot\|_{\infty})+6\delta.
	\end{align}
	We next derive the relation between the covering numbers of $\cH$ and $\cF_{\rm NN}^{\scrG}$. For any $h,h'\in \cH$, we have
	\begin{align}
		h(\xb)=\|\scrG(\xb)-\pi(\xb)\|_2^2, \mbox{ and } h'=\|\scrG'(\xb)-\pi(\xb)\|_2^2
	\end{align}
	for some $\scrG,\scrG'\in \cF_{\rm NN}^{\scrG}$. We can compute
	\begin{align}
		\|h-h'\|_{\infty}=&\sup_{\xb} \left| \|\scrG(\xb)-\pi(\xb)\|_2^2-\|\scrG'(\xb)-\pi(\xb)\|_2^2\right| \nonumber\\
		=& \sup_{\xb} \left| \left\langle \scrG(\xb)-\scrG'(\xb), \scrG(\xb)+\scrG'(\xb)-2\pi(\xb) \right\rangle \right| \nonumber\\
		\leq & \sup_{\xb} \left\| \scrG(\xb)-\scrG'(\xb)\right\|_2 \left\| \scrG(\xb)+\scrG'(\xb)-2\pi(\xb)\right\|_2 \nonumber\\
		\leq & \sqrt{4D}B\sup_{\xb} \sqrt{D} \left\| \scrG(\xb)-\scrG'(\xb)\right\|_{\infty} \nonumber\\
		=& 2DB\left\| \scrG(\xb)-\scrG'(\xb)\right\|_{L^{\infty,\infty}}.
	\end{align}
	Therefore, we have
	\begin{align}
		\cN(\delta,\cH, \|\cdot\|_{\infty})\leq \cN\left( \frac{\delta}{2DB},\cF_{\rm NN}^{\scrG},\|\cdot\|_{L^{\infty,\infty}} \right)
	\end{align}
	and
	\begin{align}
		{\rm T_2} \leq \frac{35DB^2}{n}\log \cN\left( \frac{\delta}{2DB},\cF_{\rm NN}^{\scrG},\|\cdot\|_{L^{\infty,\infty}} \right)+6\delta.
	\end{align}
\end{proof}

\subsection{Proof of Lemma \ref{lem.FG.covering}}
\label{proof.FG.covering}
\begin{proof}[Proof of Lemma \ref{lem.FG.covering}]
	We first show that there exists a network architecture $\cF(D,D;L,p,K,\kappa,R)$ so that any $\scrG\in \cF_{\rm NN}^{\scrG}$ can be realized by a network with such an architecture. Then the covering number of $\cF_{\rm NN}^{\scrG}$ can be bounded by that of $\cF(D,D;L,p,K,\kappa,R)$. 
	
	For any $\scrG\in \cF_{\rm NN}^{\scrG}$, there exist $\scrE\in \cF_{\rm NN}^{\scrE}$ and $\scrD\in \cF_{\rm NN}^{\scrD}$ so that $\scrG=\scrD\circ \scrE$. Denote the set of weights and biases of $\scrE$ by $\{(W^{\scrE}_{k},\bbb^{\scrE}_k)\}_{k=1}^{L_{\scrE}}$ and the set of weights and biases of $\scrD$ by $\{(W^{\scrD}_{k},\bbb^{\scrD}_k)\}_{k=1}^{L_{\scrD}}$. We construct $F$ as 
	\begin{align}
		F(\xb)=F_3\circ F_2\circ F_1(\xb)
	\end{align}
	where 
	\begin{align}
		&F_1(\xb)=\ReLU\left( W^{\scrE}_{L-1}\cdots \ReLU(W^{\scrE}_1\xb+\bbb_1)+ \cdots +\bbb^{\scrE}_{L-1}\right)
	\end{align}
	consists of the first $L-1$ layers of $\scrE$,
	\begin{align}
		&F_3(\tb)=W_{L_{\scrD}}^{\scrD}\cdot \ReLU\left( W^{\scrD}_{L_{\scrD}-1}\cdots \ReLU(W^{\scrD}_2\tb+\bbb_2)+ \cdots +\bbb^{\scrE}_{L-1}\right)+ \bbb^{\scrD}_{L_{\scrD}}
	\end{align}
	consists of the $2$ to $L_{\scrD}$ layers of $\scrD$.
	
	Note that we have 
	\begin{align}
		G(\xb)=F_3\circ\ReLU(W_1^{\scrD}\cdot (W_{L_{\scrE}}^{\scrE}\cdot F_1(\xb)+\bbb^{\scrE}_{L_{\scrE}})+\bbb^{\scrD}_{1})
	\end{align}
	
	We will design $F_2$ to realize the connection between $F_1$ and $F_3$ in $\scrG$ while keeping similar order of the number of parameters. We construct $F_2$as
	\begin{align}
		F_2(\rb)=\ReLU\left(\begin{bmatrix}
			W_1^{\scrD} & -W_1^{\scrD}
		\end{bmatrix} \cdot
		\ReLU\left( \begin{bmatrix}
			W^{\scrE}_{L_{\scrE}}\\ -W^{\scrE}_{L_{\scrE}}
		\end{bmatrix} \cdot \rb + 
		\begin{bmatrix}
			\bbb^{\scrE}_{L_{\scrE}} \\ -\bbb^{\scrE}_{L_{\scrE}}
		\end{bmatrix}\right)+ \bbb_1^{\scrD}\right).
	\end{align}
	Here $F_2$ is a two-layer network, with width of $O(\max\{p_{\scrD},p_{\scrE}\})$, number of nonzero parameters of $O(\max\{K_{\scrD},K_{\scrE}\})$, and all parameters are bounded by $\max\{\kappa_{\scrE},\kappa_{\scrD}\}$. Furthermore, we have 
	\begin{align}
		&F_2\circ F_1(\xb)=\ReLU(W_1^{\scrD}\cdot (W_{L_{\scrE}}^{\scrE}\cdot F_1(\xb)+\bbb^{\scrE}_{L_{\scrE}})+\bbb^{\scrD}_{1}),\\
		&F_3\circ F_2 \circ F_1(\xb)=G(\xb).
	\end{align}
	We next quantify the network size:
	\begin{itemize}
		\item $F_1$ has depth $L_{\scrE}-1$, width $dp_{\scrE}$, number of weight parameters no more than $K_{\scrE}$, and all parameters are bounded by $\kappa_{\scrE}$.
		\item $F_2$ has depth $2$, width $\max\{p_{\scrE},p_{\scrD}\}$, number of weight parameters is bounded by two times the number of parameters in $(W^{\scrE}_{L_{\scrE}},\bbb^{\scrE}_{L_{\scrE}})$ and $(W^{\scrD}_1,\bbb^{\scrD}_1)$, and all parameters are bounded by $\max\{\kappa_{\scrE},\kappa_{\scrD}\}$.
		\item $F_3$ has depth $L_{\scrD}-1$, width $Dp_{\scrD}$, number of weight parameters no more than $K_{\scrD}$, and all parameters are bounded by $\kappa_{\scrD}$.
	\end{itemize}
	In summary, $F\in \cF_{\rm NN}^F=\cF_{\rm NN}(D,D;L,p,K,\kappa,R)$ with
	\begin{align}
		L=L_{\scrE}+L_{\scrD}, p=\max\{p_{\scrE},p_{\scrD}\}, \ K=2(K_{\scrE}+K_{\scrD}), \ \kappa=\max\{\kappa_{\scrE},\kappa_{\scrD}\}, \ R=B.
	\end{align}
	Therefore $\cF_{\rm NN}^{\scrG}\subset \cF_{\rm NN}^F$ and 
	\begin{align}
		\cN\left( \delta,\cF_{\rm NN}^{\scrG},\|\cdot\|_{L^{\infty,\infty}} \right) \leq \cN\left( \delta,\cF_{\rm NN}^F,\|\cdot\|_{L^{\infty,\infty}} \right).
		\label{eq.cF.cover.compare}
	\end{align}
	
	We next derive an upper bound for $\cN\left( \delta,\cF_{\rm NN}^F,\|\cdot\|_{L^{\infty,\infty}} \right)$. We will use the following lemma:
	\begin{lemma}\label{lem.covering}
		Let $\cF_{\rm NN}(d_1,d_2;L,p,K,\kappa,R)$ be a class of network: $[-B,B]^{d_1}\rightarrow [-R,R]^{d_2}$. For any $\delta>0$, the $\delta$--covering number of $\cF_{\rm NN}(d_1,d_2;L,p,K,\kappa,R)$ is bounded by
		\begin{align}
			\cN\left( \delta,\cF_{\rm NN}(d_1,d_2;L,p,K,\kappa,R),\|\cdot\|_{L^{\infty,\infty}} \right) \leq \left( \frac{2L^2(pB+2)\kappa^L p^{L+1}}{\delta}\right)^{K}.
		\end{align}	
	\end{lemma}
	Lemma \ref{lem.covering} can be proved by following the proof of \cite[Lemma 6]{chen2019nonparametric}.
	
	By (\ref{eq.cF.cover.compare}) and Lemma \ref{lem.covering}, we have
	\begin{align}
		&\cN\left( \delta,\cF_{\rm NN}^{\scrG},\|\cdot\|_{L^{\infty,\infty}} \right) \leq \nonumber\\
		&\left( \frac{2(L_{\scrE}+L_{\scrD})^2(\max\{p_{\scrE},p_{\scrD}\}B+2)(\max\{\kappa_{\scrE},\kappa_{\scrD}\})^{L_{\scrE}+L_{\scrD}} (\max\{p_{\scrE},p_{\scrD}\})^{L_{\scrE}+L_{\scrD}+1}}{\delta}\right)^{2(K_{\scrE}+K_{\scrD})}.
	\end{align}

\end{proof}

\subsection{Proof of Lemma \ref{lem.multi.approx}}
\label{proof.multi.approx}
\begin{proof}[Proof of Lemma \ref{lem.multi.approx}]
	We first show that there exist an encoder $\scrE:\cM(q)\rightarrow \RR^{C_{\cM}(d+1)}$ and a decoder $\scrD:\RR^{C_{\cM}(d+1)} \rightarrow \cM$ satisfying
	\begin{align}
		\scrD\circ\scrE(\xb)=\pi(\xb)
  \label{eq.ec}
	\end{align} 
 for any $\xb\in \cM(q)$.
	 We call $\scrE(\xb)$ and $\scrD$ as the oracle encoder and decoder. Then we show that there exists networks $\widetilde{\scrE}$ and $\widetilde{\scrD}$ approximating $\scrE$ and $\scrD$ so that $\widetilde{\scrD}\circ\widetilde{\scrE}$ approximates $\scrD\circ\scrE$ with high accuracy.
	
	\noindent $\bullet$ {\bf Constructing $\scrE$ and $\scrD$.} 
	The construction of $\scrE$ and $\scrD$ relies on a proper construction of an atlas of $\cM$ and a partition of unity of $\cM(q)$. We construct $\scrE$ and $\scrD$ using the following three steps.

 {\bf Step 1}. In the first step, we use the results from  \cite{cloninger2021deep} to construct a partition of unity of $\cM(q)$. 

  Define the local reach \citep{boissonnat2010manifold} of $\cM$ at $\vb\in \cM$ as
\begin{align}
	\tau_{\cM}(
 \vb)= \inf_{\xb\in G} \|\xb-\vb\|_2.
\end{align}
where $G= \left\{\xb\in \RR^D: \exists \mbox{ distinct } \pb,\qb\in \cM \mbox{ such that } d(\xb,\cM)=\|\xb-\pb\|_2=\|\xb-\qb\|_2\right\}$ is the medial axis of $\cM$.
We have 
$$
\tau_{\cM}=\inf_{\vb\in\cM} \tau_{\cM}(\vb).
$$
Let $\{\vb_j'\}_{j=1}^c$ be a $\delta$--separated set of $\cM$ for some integer $c>0$. Define $p=\frac{1}{2}(1+q/\tau)$, $h=\frac{6}{1-q/(p\tau)}$. If $\delta$ satisfies
\begin{align}
    \delta<C(1-q/\tau)^2\tau
\end{align}
for some absolute constant $C$,
\citet[Proposition 6.3]{cloninger2021deep} constructs a partition of unity of $\cM(q)$, denoted by $\{\eta_j\}_{j=1}^c$, defined as
 \begin{align}
    \bar{\eta}_j(\xb)=\max\left\{ \left( 1-\left( \frac{\|\xb-\vb_j'\|_2}{p\tau_{\cM}(\vb_j')} \right) ^2 -\left( \frac{\|A_{\vb_j'}^{\top}(\xb-\vb_j')\|_2}{h\delta}\right)^2\right), 0 \right\}, \quad \eta_j(\xb)=\frac{\bar{\eta}(\xb)}{\|\bar{\eta}(\xb)\|_1},
    \label{eqetaj}
\end{align}
 where $\bar{\eta}=[\bar{\eta}_1,...,\bar{\eta}_{c}]^{\top}$, 
  and $A_{\vb_j'}$ denotes the $D\times d$ matrix containing columnwise orthonormal basis for the tangent space $\cM$ at $\vb_j'$.  In (\ref{eqetaj}), $j$ is the index of each component of this partition of unity, and the construction of $\{\eta_j\}_{j=1}^c$ only depends on the $\delta$--separated set $\{\vb_j'\}_{j=1}^c$ and properties of $\cM$. With this construction, the cardinality of $\{\eta_j\}_{j=1}^c$ (and thus the corresponding atlas of $\cM$) depends on $q$, which goes to infinity as $q$ approaches $\tau$. 
  
  {\bf Step 2.} In the second step, we apply a grouping technique to $\{\eta_j\}_{j=1}^c$ to construct a partition of unity of $\cM(q)$ and an atlas of $\cM$  so that their cardinality only depends on $\cM$ itself.

 We use the following lemma:
	\begin{lemma}\label{lem.atlas}
		For any $\cM$ in Setting \ref{setting}, there exists two atlases $\{\widetilde{V}_j,\phi_j\}_{j=1}^{C_{\cM}}$ and $\{\bar{V}_j,\phi_j\}_{j=1}^{C_{\cM}}$ with $C_{\cM}=O((d\log d) (4/\tau)^d)$ so that for each $j=1,...,C_{\cM}$, it holds
		\begin{itemize}
			\item[(i)] $\bar{V}_j\subset \widetilde{V}_j$.
			\item[(ii)] 
			$
			\inf\limits_{\bar{\vb}\in \bar{V}_j,\widetilde{\vb}\in \partial \widetilde{V}_j} d_{\cM}(\bar{\vb},\widetilde{\vb})>\tau/8,
			$
			where $d_{\cM}(\cdot,\cdot)$ denotes the geodesic distance on $\cM$.
			\item[(iii)] 
   We have 
			\begin{align*}
				&\|\phi_j(\vb_1)-\phi_j(\vb_2)\|_{\infty}\leq \|\vb_1-\vb_2\|_2\leq d_{\cM}(\vb_1,\vb_2),\\
				&\|\phi_j(\vb_1)\|_{\infty}\leq \tau/4,
			\end{align*}
   for any $\vb_1,\vb_2\in \widetilde{V}_j$.

			\item[(iv)] For any $j=1,...,C_{\cM}$ and $\zb_1,\zb_2\in \phi_j(\widetilde{V}_j)$, we have
			\begin{align*}
				&\|\phi_j^{-1}(\zb_1)-\phi_j^{-1}(\zb_2)\|_{\infty}\leq 2 \|\zb_1-\zb_2\|_2.
			\end{align*}
		\end{itemize}
	\end{lemma}
Lemma \ref{lem.atlas} is proved in Appendix \ref{proof.atlas}.

	  We next construct an atlas $\{V_j,\phi_j\}_{j=1}^{C_\cM}$ of $\cM$ using Lemma \ref{lem.atlas}, as well as
  a partition of unity $\{\rho_j\}_{j=1}^{C_{\cM}}$ defined on $\cM(q)$ so that $\{\rho_j|_{\cM}\}_{j=1}^{C_{\cM}}$ is a partition of unity subordinate to $\{V_j,\phi_j\}_{j=1}^{C_\cM}$.

Let $U_j$ be the support of $\eta_j$ defined in \eqref{eqetaj}. Then $\{U_j\}_{j=1}^c$ forms a cover of $\cM(q)$. According to \citet[Proposition 6.3]{cloninger2021deep}, we have
	\begin{align}
		(U_j\cap \cM)\subset B_{\cM,r}(\vb_j') \quad \mbox{ with } \quad r=c_1\delta
		\label{eq.U.Mrad}
	\end{align}
	for some constant $c_1$ depending on $q$ and $\tau_{\cM}$, where $B_{\cM,r}(\vb)$ denotes the geodesic ball on $\cM$ centered at $\vb$ with radius $r$.
	
	For $j=1,...,C_{\cM}$, we sequentially construct
 \begin{align*}
		\cI_j=\{k: (U_k\cap \cM)\cap \bar{V}_j \neq \emptyset \text{ and } k\notin \cI_{j'} \mbox{ for all } j'<j\}, \quad V_j=\bigcup\limits_{k\in \cI_j} (U_k\cap \cM).
	\end{align*}
 Such a construction ensures that $k$ only belongs to one $I_j$.
	According to (\ref{eq.U.Mrad}) and Lemma \ref{lem.atlas}(ii), as long as $2c_1\delta<\tau/8$, we have $V_j\subset \widetilde{V}_j$. Therefore $\phi_j$ is well defined on $V_j$. As a result, $\{V_j,\phi_j\}_{j=1}^{C_{\cM}}$ is an atlas of $\cM$. The relation among $\bar{V}_j,\widetilde{V}_j,V_j$ and $V_k$'s is illustrated in Figure \ref{fig:atlas}.
 \begin{figure}[t!]
     \centering
     \includegraphics[width=0.6\textwidth]{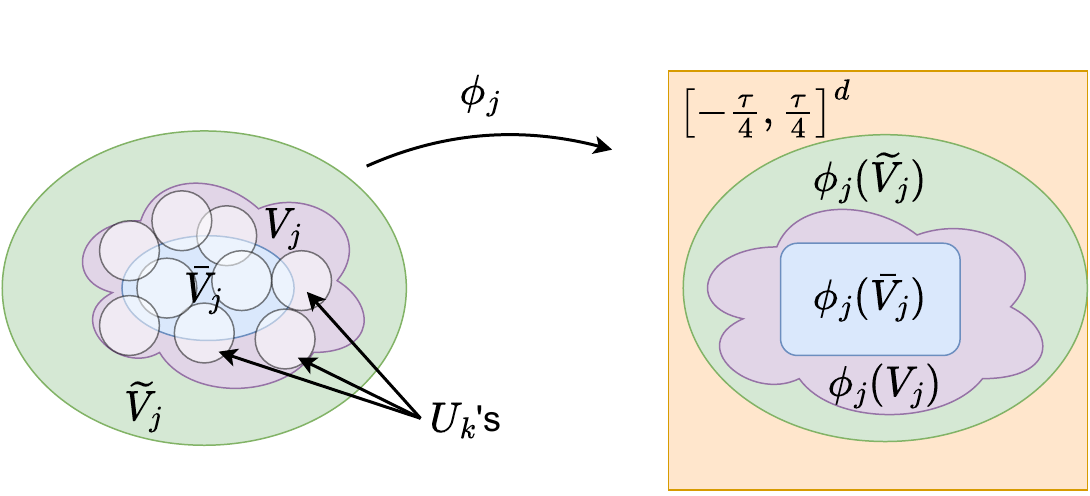}
     \caption{Illustration of the relation among $\bar{V}_j,\widetilde{V}_j,V_j$ and $U_k$'s. In the left figure, $\bar{V}_j\subset \widetilde{V}_j$ and $\partial\bar{V}_j$ is away from $\partial \widetilde{V}_j$ according to Lemma \ref{lem.atlas}. For all $U_k$'s (white regions) that intersect $\bar{V}_j$ (blue region), they are inside $\widetilde{V}_j$ (green region). The set $V_j$ (purple region) is the union of all such $U_k$'s that are not included in $V_{j'}$ for $j'<j$. After transformation $\phi_j$, the relation among these sets are illustrated in the right figure.}
     \label{fig:atlas}
 \end{figure}
	
	Define 
	\begin{align}
		\rho_j(\xb)=\sum_{k\in \cI_j} \eta_k(\xb).
		\label{eq.rho}
	\end{align}
	Then $\{\rho_j\}_{j=1}^{C_{\cM}}$ is a partition of unity of $\cM(q)$ and $\{\rho_j|_{\cM}\}_{j=1}^{C_{\cM}}$ is a partition of unity of $\cM$ subordinate to $\{V_j\}_{j=1}^{C_{\cM}}$.  
	
	We define the oracle encoder $\scrE: \cM(q)\rightarrow \RR^{C_{\cM}(d+1)}$ as
	\begin{align}
		\scrE(\xb)=\begin{bmatrix}
			\fb_1(\xb) & \cdots & \fb_{C_{\cM}}(\xb)
		\end{bmatrix}^{\top}, \quad 
		\mbox{ with } \quad \fb_j(\xb)=\begin{bmatrix}
			(\phi_j(\pi(\xb)))^{\top} & \rho_j(\xb)
		\end{bmatrix}\in \RR^{d+1},
		\label{eq.ecoder.multi}
	\end{align}
	and the corresponding decoder $\scrD: \RR^{C_{\cM}(d+1)}\rightarrow \cM$ as
	\begin{align*}
		\scrD(\zb)=\sum_{j=1}^{C_{\cM}} \phi_j^{-1}((\zb_j)_{1:d})\times (\zb_j)_{d+1},
	\end{align*}
	where $(\zb_j)_{1:d}=\begin{bmatrix}
		(\zb_j)_1 & \cdots & (\zb_j)_d 
	\end{bmatrix}.$
	For any $\xb\in \cM(q)$, we can verify that
	\begin{align*}
		\scrD\circ\scrE(\xb)=&\sum_{j=1}^{C_{\cM}} \phi_j^{-1}\circ\phi_j(\pi(\xb))\times \rho_j(\xb) \nonumber\\
		=&\sum_{\xb\in \supp(\rho_j)} \pi(\xb)\times \rho_j(\xb) \nonumber\\
		=& \pi(\xb).
	\end{align*}
	
	\noindent $\bullet$ {\bf Constructing $\widetilde{\scrE}$ and $\widetilde{\scrD}$.}
	
	The following lemma show that there exist network $\widetilde{\scrE}$ approximating $\scrE$ and network $\widetilde{\scrD}$ approximating $\scrD$ so that $\widetilde{\scrG}=\widetilde{\scrD}\circ\widetilde{\scrE}$ approximate $\scrG$ with high accuracy.
	
	\begin{lemma}\label{lem.multi.approx.G}
		Consider Setting \ref{setting}. For any $0<\varepsilon<\tau/2$, there exists a network architecture $\cF_{\rm NN}^{\scrE}=\cF(D,C_{\cM}(d+1);L_{\scrE}, p_{\scrE}, K_{\scrE},\kappa_{\scrE},R_{\scrE})$ giving rise to a network $\widetilde{\scrE}$, and a network architecture $\cF_{\rm NN}^{\scrD}=\cF(C_{\cM}(d+1),D;L_{\scrD}, p_{\scrD}, K_{\scrD},\kappa_{\scrD},R_{\scrD})$ giving rise to a network $\widetilde{\scrD}$, so that 
  \begin{align*}
			\sup_{\xb\in \cM(q)}\|\widetilde{\scrG}(\xb)-\scrD\circ\scrE(\xb)\|_{\infty}\leq \varepsilon
		\end{align*}
		with $\widetilde{\scrG}=\widetilde{\scrD}\circ\widetilde{\scrE}$.
		The network architectures have
		\begin{align*}
			&L_{\scrE}=O(\log^2 \varepsilon^{-1}+\log D), \ p_{\scrE}=O(D\varepsilon^{-d}), \ K_{\scrE}=O((D\log D)\varepsilon^{-d}\log^2 \varepsilon^{-1}), \nonumber\\
			& \kappa_{\scrE}=O(\varepsilon^{-2}), \ R_{\scrE}=\max\{\tau/4,1\}.
		\end{align*} 
		and
		\begin{align*}
			&L_{\scrD}=O(\log^2 \varepsilon^{-1}+\log D), \ p_{\scrD}=O(D\varepsilon^{-d}), \ K_{\scrD}=O(D\varepsilon^{-d} \log^2 \varepsilon +D\log D), \nonumber\\
			& \kappa_{\scrD}=O(\varepsilon^{-1}), \ R_{\scrD}=B.
		\end{align*}
	\end{lemma}
	Lemma \ref{lem.multi.approx.G} is proved in Appendix \ref{proof.multi.approx.G}.
	
	Let $\widetilde{\scrE}$ and $\widetilde{\scrD}$ be defined in Lemma \ref{lem.multi.approx.G} so that
	\begin{align*}
		\sup_{\xb\in \cM(q)}\|\widetilde{\scrG}(\xb)-\scrD\circ\scrE(\xb)\|_{\infty}\leq \varepsilon
	\end{align*}
	For any $\xb\in \cM(q)$, we have 
	\begin{align*}
		\|\widetilde{\scrG}(\xb)-\pi(\xb)\|_{\infty}= \|\widetilde{\scrG}(\xb)-\scrD\circ\scrE(\xb)\|_{\infty}\leq \varepsilon.
	\end{align*}
\end{proof}

\subsection{Proof of Lemma \ref{lem.atlas}}
\label{proof.atlas}
\begin{proof}[Proof of Lemma \ref{lem.atlas}]
	We construct $\{\bar{V}_j,\phi_j\}_{j=1}^{C_{\cM}}$ and $\{\widetilde{V}_j,\phi_j\}_{j=1}^{C_{\cM}}$ by covering $\cM$ using Euclidean balls. We first use Euclidean balls with radius $r_1=\tau/8$ to cover $\cM$. Since $\cM$ is compact, the number of balls is finite. Denote the number of balls by $C_{\cM}$ and the centers by $\{\cbb_j\}_{j=1}^{C_{\cM}}$. We define 
	\begin{align*}
		\bar{V}_j=B_{r_1}(\cbb_j)\cap \cM.
	\end{align*}
Then $\{\bar{V}_j\}_{j=1}^{C_{\cM}}$ is a cover of $\cM$. 
The following lemma shows that $C_{\cM}$ is a constant depending on $d$ and $\tau$.
\begin{lemma}\label{lem.Mcover}
    Let $\cM$ be a $d$-dimensional compact Riemannian manifold embedded in $\RR^D$. Assume $\cM$ has reach $\tau>0$. Let $\{V_j\}_{j=1}^{C_{\cM}}$ be a minimum cover of $\cM$ with $V_j=\cM\cap B_r(\cbb_j)$ for a set of centers $\{\cbb_j\}_{j=1}^{C_{\cM}}$. For any $r<\tau/2$, we have
    \begin{align}
        C_{\cM}\leq \frac{|\cM|}{\cos^d(\arcsin \frac{r}{2\tau}) |B_r^d|},
    \end{align}
    where $|\cM|$ denotes the volume of $\cM$, and $|B_r^d|$ denotes the volume of the $d$-dimensional Euclidean ball with radius $r$.  
\end{lemma}
Lemma \ref{lem.Mcover} is proved in Appendix \ref{proof.Mcover}. 
By Lemma \ref{lem.Mcover}, $C_{\cM}$ is proportional to $|\cM|/r^d$.

Let $\bar{P}_j$ be the orthogonal projection from $\bar{V}_j$ to the tangent plane of $\cM$ at $\cbb_j$. By \citet[Lemma 4.2]{chen2019nonparametric}, $\bar{V}_j$ is diffeomorphic to a subset of $\RR^d$ and $\bar{P}_j$ is a diffeomorphism.

	With the same set of centers, we use Euclidean balls with radius $r_2=\tau/4$ to cover $\cM$. Define 
	\begin{align*}
		\widetilde{V}_j=B_{r_2}(\cbb_j)\cap \cM
	\end{align*}
	and $\widetilde{P}_j$ be the orthogonal projection from $\widetilde{V}_j$ to the tangent plane of $\cM$ at $\cbb_j$. As $\widetilde{V}_j$ is bounded and $C_{\cM}$ is finite, there exists a constant $\Lambda$ depending on $\tau$ so that $\|\widetilde{P}_j(\vb)\|_{\infty}\leq \Lambda$. Setting $\widetilde{P}_j$ so that $\widetilde{P}_j(\cbb_j)=\mathbf{0}$, we have $\Lambda\leq \tau/4$. Again by \citet[Lemma 4.2]{chen2019nonparametric}, $\widetilde{P}_j$ is a diffeomorphism between $\widetilde{V}_j$ and $\RR^d$.
	We set $\phi_j=\widetilde{P}_j$. Since $\bar{V}_i\subset \widetilde{V}_i$, $\phi_i$ is well defined on $\bar{V}_i$. We have
	\begin{align*}
		\inf\limits_{\bar{\vb}\in \bar{V}_i,\widetilde{\vb}\in \partial \widetilde{V}_i} d_{\cM}(\bar{\vb},\widetilde{\vb})\geq \inf\limits_{\bar{\vb}\in \bar{V}_i,\widetilde{\vb}\in \partial \widetilde{V}_i} \|\bar{\vb}-\widetilde{\vb}\|_2\geq r_2-r_1=\tau/8.
	\end{align*}
	Since $\phi_j$'s are diffeomorphisms, there exist constants $C_f$ and $C_g$ so that item (iii) and (iv) hold.

	We next prove the Lipschitz property of $\phi_j$ and $\phi_j^{-1}$. For $\phi_j$ and any $\vb_1,\vb_2\in \widetilde{V}_j$, we have
	\begin{align*}
		\|\phi_j(\vb_1)-\phi_j(\vb_2)\|_{\infty}= \|\widetilde{P}_j(\vb_1-\vb_2)\|_{\infty} \leq \|\widetilde{P}_j(\vb_1-\vb_2)\|_{2}\leq \|\vb_1-\vb_2\|_2\leq d_{\cM}(\vb_1,\vb_2).
	\end{align*}
	We then focus on $\phi^{-1}$. Denote the tangent space of $\cM$ at $\vb\in\cM$ by $T_{\vb}\cM$, and the principal angle between two tangent spaces $T_{\vb_1}\cM,T_{\vb_2}\cM$ by $\angle(T_{\vb_1}\cM,T_{\vb_2}\cM)$. We will use the following lemma
	\begin{lemma}[Corollary 3 in \cite{boissonnat2019reach}]\label{lem.geometry}
		Let $\cM$ be a $d$-dimensional manifold embedded in $\RR^D$. Denote the reach of $\cM$ by $\tau$. We have
		\begin{align*}
			\sin \frac{\angle(T_{\vb_1}\cM,T_{\vb_2}\cM)}{2}\leq \frac{\|\vb_1-\vb_2\|_2}{2\tau}
		\end{align*}
	for any $\vb_1,\vb_2\in \cM$.
	\end{lemma}

We have
\begin{align}
	\|\phi_j^{-1}(\zb_1)-\phi_j^{-1}(\zb_2)\|_{\infty}\leq \|\phi_j^{-1}(\zb_1)-\phi_j^{-1}(\zb_2)\|_{2} =\|\vb_1-\vb_2\|_2=\frac{\|\zb_1-\zb_2\|_2}{|\cos(\angle(\vb_1-\vb_2,T_{\cbb_j}\cM))|},
	\label{eq.vbzb}
\end{align}
where the last equality is due to $\phi_j$ being an orthogonal projection and some geometric derivation, see Figure \ref{fig.angle}  for an illustration.
\begin{figure}[t!]
	\centering
	\includegraphics[width=0.6\textwidth]{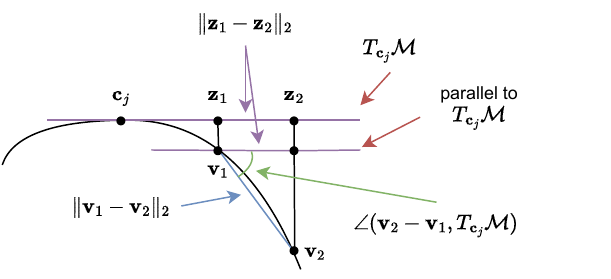}
	\caption{ Illustration of the last equality in (\ref{eq.vbzb}).}
	\label{fig.angle} 
\end{figure}

The denominator in (\ref{eq.vbzb}) can be lower bounded as
\begin{align}
	|\cos(\angle(\vb_1-\vb_2,T_{\cbb_j}\cM))|=&\sqrt{1-\sin^2(\angle(\vb_1-\vb_2,T_{\cbb_j}\cM))} \nonumber\\
	\geq&  \sqrt{1-\max_{\vb\in \widetilde{V}_j} \sin^2(\angle(T_{\vb}\cM,T_{\cbb_j}\cM))} \nonumber\\
	\geq & \sqrt{1-\max_{\vb\in \widetilde{V}_j} 4\sin^2\left(\frac{\angle(T_{\vb}\cM,T_{\cbb_j}\cM)}{2}\right)} \nonumber\\
	\geq & \sqrt{1-\max_{\vb\in \widetilde{V}_j} 4\left(\frac{\|\vb-\cbb_j\|_2}{2\tau}\right)^2} \nonumber\\
	= & \sqrt{1-4\left(\frac{\tau}{4\cdot 2\tau}\right)^2} \nonumber\\
	= & \frac{\sqrt{15}}{4} \nonumber\\
	\geq & 1/2,
	\label{eq.cos}
\end{align}
where in the third inequality Lemma \ref{lem.geometry} is used.

Substituting (\ref{eq.cos}) into (\ref{eq.vbzb}) gives rise to
\begin{align*}
	\|\phi_j^{-1}(\zb_1)-\phi_j^{-1}(\zb_2)\|_{\infty}\leq 2\|\zb_1-\zb_2\|_2.
\end{align*}

\end{proof}

\subsection{Proof of Lemma \ref{lem.multi.approx.G}}
\label{proof.multi.approx.G}
\begin{proof}[Proof of Lemma \ref{lem.multi.approx.G}]
	We will construct networks $\widetilde{\scrE}$ and $\widetilde{\scrD}$ to approximate $\scrE$ and $\scrD$, respectively. 
	
	\noindent $\bullet$ {\bf Construction of $\widetilde{\scrE}$.}
	
	The encoder $\scrE$ is a collection of $\fb_j$'s, which consist of $\phi_j,\pi$ and $\rho_j$. We show that these functions can be approximated well be networks. 
	By Lemma \ref{lem.atlas}(iii), $\phi_j$ is Lipschitz continuous with Lipschitz constant $1$. According to Lemma \ref{lem.f}, for any $0<\varepsilon<1$, there exists a network architecture $\cF(D,d;L_1,p_1,K_1,\kappa_1,R_1)$ that gives rise to a network $\widetilde{\phi}_j$ satisfying
	\begin{align}
		\sup_{\xb\in V_j}\|\widetilde{\phi}_j(\xb)-\phi_j\circ\pi(\xb)\|_{\infty}\leq \varepsilon_1.
		\label{eq.multi.phi}
	\end{align}
	Such an architecture has
	\begin{align*}
		L_1=O\left(\log^2 \varepsilon_1^{-1}+\log D\right), \ p_1=O\left(D\varepsilon_1^{-d}\right), \ K_1=O\left(D\varepsilon_1^{-d}\log^2 \varepsilon_1+D\log D\right),\ \kappa_1=O\left(\varepsilon_1^{-2}\right), \ R_1=\tau/4.
	\end{align*}

	The following lemma shows that $\rho_j$ can be approximated by a network with arbitrary accuracy (see a proof in Appendix \ref{proof.rho}):
	\begin{lemma}\label{lem.rho}
		Consider Setting \ref{setting}. For any $0<\varepsilon<1$ and $j=1,...,C_{\cM}$, there exists a network architecture $\cF(D,1;L,p,K,\kappa,1)$ giving rise to a network $\widetilde{\rho}_j$ so that
		\begin{align*}
			\sup_{\xb\in \cM(q)}|\widetilde{\rho}_j(\xb)-\rho_j(\xb)|\leq \varepsilon.
		\end{align*}
		Such a network architecture has
		\begin{align*}
			L=O(\log^2 \varepsilon^{-1} + \log D), \ p=O(D\varepsilon^{-1}), \ K=O((D\log D)\varepsilon^{-1}\log^2 \varepsilon^{-1}), \kappa=O(\varepsilon^{-2}).
		\end{align*}
	\end{lemma}

	By Lemma \ref{lem.rho}, there exists a network $\widetilde{\rho}_j\in \cF(D,1;L_2,p_2,K_2,\kappa_2,R_2)$ so that
	\begin{align}
  \sup_{\xb\in \cM(q)}|\widetilde{\rho}_j(\xb)-\rho_j(\xb)| \leq \varepsilon_1.
		\label{eq.multi.rho}
	\end{align}
	Such a network architecture has
	\begin{align*}
		L_2=O(\log^2 \varepsilon_1^{-1} + \log D), \ p_2=O(D\varepsilon_1^{-1}), \ K_2=O((D\log D)\varepsilon_1^{-1}\log^2 \varepsilon_2^{-1}), \kappa_2=O(\varepsilon_1^{-2}), R_2=1.
	\end{align*}
	
	Define 
	\begin{align}
		\widetilde{\fb}_j(\xb)=\begin{bmatrix}
			(\widetilde{\phi}_j(\xb))^{\top} & \widetilde{\rho}_j(\xb) 
		\end{bmatrix}^{\top}
		\quad \mbox{ and } \quad \widetilde{\scrE}(\xb) = \begin{bmatrix}
			\widetilde{\fb}_1(\xb) & \cdots & \widetilde{\fb}_{C_{\cM}}(\xb)
		\end{bmatrix}^{\top}.
		\label{eq.ftilde}
	\end{align}
	We have $\widetilde{\scrE}\in \cF^{\scrE}(D,C_{\cM}(d+1);L_{\scrE}, p_{\scrE}, K_{\scrE}, \kappa_{\scrE}, R_{\scrE})$ with
	\begin{align*}
		&L_{\scrE}=O(\log^2 \varepsilon_1^{-1}+\log D), \ p_{\scrE}=O(D\varepsilon_1^{-1}), \ K_{\scrE}=O((D\log D)\varepsilon_1^{-d}\log^2 \varepsilon_1^{-1}), \nonumber\\
		& \kappa_{\scrE}=O(\varepsilon_1^{-2}), \ R_{\scrE}=\max\{\tau/4,1\}.
	\end{align*}
	The constant hidden in $O$ depends on $d,\tau,q,B,C_{\cM}$ and the volume  of $\cM$.

	\noindent $\bullet$ {\bf Construction of $\widetilde{\scrD}$.}
	
	By Lemma \ref{lem.atlas}(iv), $\phi_j^{-1}$ is Lipschitz continuous with Lipschitz constant $2$. Although $\phi_j^{-1}$ is only defined on $\phi_j(\widetilde{V}_j)\subset [-\tau/4,\tau/4]^d$, we can extend it to $[-\tau/4,\tau/4]^d$ by Lemma \ref{lem:extension}.   
Such an extension preserves the property Lemma \ref{lem.atlas}(iv).

For any $0<\varepsilon_2<1$, by Lemma \ref{lem.g}, there exists a network $\bar{g}_j\in \cF(d,D;L_3,p_3,K_3,\kappa_3,R_3)$ so that
	\begin{align*}
		\sup_{\zb\in [-\tau/4,\tau/4]^d}\|\bar{g}_j(\zb)-\phi^{-1}_j(\zb)\|_{\infty}\leq \varepsilon_2.
	\end{align*} 
	Such a network architecture has 
	\begin{align*}
		L_3=O(\log^2 \varepsilon_2^{-1})+\log D, \ p_3=O(D\varepsilon_2^{-d}), \ K_3=O(D\varepsilon_2^{-d}\log^2 \varepsilon_2+D\log D),\ \kappa_3=\varepsilon_2^{-1},\ R_3=B.
	\end{align*}

	Note that the input of $\bar{g}_j$ is in $\RR^d$ while the input to the decoder is in $\RR^{C_{\cM}(d+1)}$. We will append $\bar{g}_j$ by a linear transformation that extract the corresponding elements (i.e., the first $d$ element in the output of $\widetilde{\fb}_j$) from the input of the decoder. Define the matrix $A_j\in \RR^{d\times(C_{\cM}(d+1))}$ with
	\begin{align*}
		(A_j)_{i,k}=\begin{cases}
			1 & \mbox{ if }  k=(j-1)(d+1)+i,\\
			0 & \mbox{ otherwise}.
		\end{cases}
	\end{align*}
	We define the network 
	\begin{align*}
		\widetilde{g}_j(\zb)=\bar{g}_j(A_j\zb)
	\end{align*}
	for any $\zb\in \RR^{C_{\cM}(d+1)}$.
	
	The following lemma shows that the multiplication operator $\times$ can be well approximated by networks:
	\begin{lemma}[Proposition 3 in \cite{yarotsky2017error}]\label{lem.multiplication0}
		For any $B>0$ and $0<\varepsilon<1$, there exists a network $\widetilde{\times}$ so that for any $|x_1|\leq B$ and $|x_2|\leq B$, we have
		\begin{align*}
			|\widetilde{\times}(x_1,x_2)-x_1\times x_2|<\varepsilon,\ \widetilde{\times}(x_1,0)=\widetilde{\times}(0,x_2)=0.
		\end{align*}
		Such a network has $O(\log \varepsilon^{-1})$ layers and parameters. The width is bounded by 6 and all parameters are bounded by $B^2$.
	\end{lemma}
	Based on Lemma \ref{lem.multiplication0}, it is easy to derive the following lemma:
	\begin{lemma}\label{lem.multiplication}
		For any $B>0$ and $0<\varepsilon<1$, there exists a network $\widetilde{\times}$ so that for any $\xb\in \RR^D$ with $\|\xb\|_{\infty}\leq B$ and $|y|\leq B$, we have
		\begin{align*}
			\|\widetilde{\times}(\xb,y)-y\xb\|_{\infty}<\varepsilon
		\end{align*}
		and $(\widetilde{\times}(\xb,y))_i=0$ if either $x_i=0$ or $y=0$, where $x_i$ denotes the $i$-th element of $\xb$.
		Such a network has $O(\log \varepsilon^{-1})$ layers, $O(D\log \varepsilon^{-1})$ parameters. The width is bounded by $6D$ and all parameters are bounded by $B^2$.
	\end{lemma}
	Lemma \ref{lem.multiplication} can be proved by stacking $D$ networks defined in Lemma \ref{lem.multiplication0}. The proof is omitted here.
	
	Let $\widetilde{\times}$ be the network defined in Lemma \ref{lem.multiplication} with accuracy $\varepsilon_2$.
	We construct the decoder $\widetilde{\scrD}$ as
	\begin{align*}
		\widetilde{\scrD}(\zb)=\sum_{j=1}^{C_{\cM}} \widetilde{\times}\left(\widetilde{g}_j(\zb),\bar{A}_j\zb\right),
	\end{align*}
	where $\bar{A}_j\in \RR^{1\times C_{\cM}(d+1)}$ is a weight matrix defined by
	\begin{align*}
		(\bar{A}_j)_{1,k}=\begin{cases}
			1 & \mbox{ if } k=j(d+1),\\
			0 & \mbox{ otherwise}.
		\end{cases}
	\end{align*}
 Let $\Omega_{\zb}=([-\tau/4,\tau/4]\times([0,1]))^{C_{\cM}}$.
	Setting $\varepsilon_2=\frac{\varepsilon_3}{2C_{\cM}}$, we deduce that 
	\begin{align}
		\sup_{\zb\in \Omega_{\zb}}\|\widetilde{\scrD}(\zb)-\scrD(\zb)\|_{\infty}\leq &\sup_{\zb\in \Omega_{\zb}}\sum_{j=1}^{C_{\cM}} \left\|\widetilde{\times}\left(\widetilde{g}_j(\zb),\bar{A}_j\zb\right) - \phi_j^{-1}((\zb_j)_{1:d})\times (\zb_j)_{d+1}\right\|_{\infty} \nonumber\\
		=&\sup_{\zb\in \Omega_{\zb}}\sum_{j=1}^{C_{\cM}} \left\|\widetilde{\times}\left(\widetilde{g}_j(\zb),\bar{A}_j\zb\right) - \phi_j^{-1}(A_j\zb)\times (\bar{A}_j\zb)\right\|_{\infty} \nonumber\\
		\leq &\sup_{\zb\in \Omega_{\zb}}\sum_{j=1}^{C_{\cM}} \Big( \left\|\widetilde{\times}\left(\widetilde{g}_j(\zb),\bar{A}_j\zb\right) - \widetilde{g}_j(\zb)\times (\bar{A}_j\zb)\right\|_{\infty}  \nonumber \\
		& +  \left\|\widetilde{g}_j(\zb)\times  (\bar{A}_j\zb)- \phi_j^{-1}(A_j\zb)\times (\bar{A}_j\zb)\right\|_{\infty}\Big) \nonumber\\
		\leq & \sum_{j=1}^{C_{\cM}} (\varepsilon_2+\varepsilon_2) \nonumber\\
		= & 2C_{\cM}\varepsilon_2 \nonumber\\
		=&\varepsilon_3.
		\label{eq.D.err}
	\end{align}
	
	\noindent $\bullet$ {\bf Error estimation of $\widetilde{\scrG}$.}
	We have 
	\begin{align}
		\sup_{\xb\in \cM(q)}\|\widetilde{\scrG}(\xb)-G(\xb)\|_{\infty}= & \sup_{\xb\in \cM(q)}\|\widetilde{\scrD}\circ\widetilde{\scrE}(\xb)-\scrD\circ\scrE(\xb)\|_{\infty} \nonumber\\
		\leq &\sup_{\xb\in \cM(q)}\|\widetilde{\scrD}\circ\widetilde{\scrE}(\xb)-\scrD\circ\widetilde{\scrE}(\xb)\|_{\infty} + \sup_{\xb\in \cM(q)}\|\scrD\circ\widetilde{\scrE}(\xb)-\scrD\circ\scrE(\xb)\|_{\infty} \nonumber\\
		\leq & \varepsilon_3+ \sup_{\xb\in \cM(q)}\|\scrD\circ\widetilde{\scrE}(\xb)-\scrD\circ\scrE(\xb)\|_{\infty}.
		\label{eq.multi.G.err.1}
	\end{align}
	We next derive an upper bound of the second term in (\ref{eq.multi.G.err.1}). Recall the definition of $\fb_j$ and $\widetilde{\fb}_j$ in (\ref{eq.ecoder.multi}) and (\ref{eq.ftilde}), respectively.
	Plugin the expression of $\scrD$ into the second term in (\ref{eq.multi.G.err.1}), we have
	\begin{align}
		\sup_{\xb\in \cM(q)}\|\scrD\circ\widetilde{\scrE}(\xb)-\scrD\circ\scrE(\xb)\|_{\infty} = &\sup_{\xb\in \cM(q)}\left\|\sum_{j=1}^{C_{\cM}} \phi_j^{-1}((\widetilde{\fb}_j)_{1:d})\times (\widetilde{\fb}_j)_{d+1} -\sum_{j=1}^{C_{\cM}} \phi_j^{-1}((\fb_j)_{1:d})\times (\fb_j)_{d+1}\right\|_{\infty} \nonumber\\
		\leq & \sup_{\xb\in \cM(q)}\sum_{j=1}^{C_{\cM}} \left\| \phi_j^{-1}((\widetilde{\fb}_j)_{1:d})\times (\widetilde{\fb}_j)_{d+1}- \phi_j^{-1}((\fb_j)_{1:d})\times (\fb_j)_{d+1} \right\|_{\infty} \nonumber\\
		\leq & \sup_{\xb\in \cM(q)}\sum_{j=1}^{C_{\cM}} \bigg(\left\| \phi_j^{-1}((\widetilde{\fb}_j)_{1:d})\times (\widetilde{\fb}_j)_{d+1}- \phi_j^{-1}((\widetilde{\fb}_j)_{1:d})\times (\fb_j)_{d+1} \right\|_{\infty} \nonumber\\
		&\hspace{1cm} + \left\| \phi_j^{-1}((\widetilde{\fb}_j)_{1:d})\times (\widetilde{\fb}_j)_{d+1}- \phi_j^{-1}((\fb_j)_{1:d})\times (\fb_j)_{d+1} \right\|_{\infty} \bigg) \nonumber\\
		\leq & \sup_{\xb\in \cM(q)}\sum_{j=1}^{C_{\cM}} \bigg(\|\phi_j^{-1}((\widetilde{\fb}_j)_{1:d})\|_{\infty}\|(\widetilde{\fb}_j)_{d+1}-(\fb_j)_{d+1}\|_{\infty} \nonumber\\
		&\hspace{1cm}  + 2\|(\fb_j)_{d+1}\|_{\infty}\|(\widetilde{\fb}_j)_{1:d}-(\fb_j)_{1:d}\|_2
		\bigg) \nonumber\\
		\leq &  \sup_{\xb\in \cM(q)}\sum_{j=1}^{C_{\cM}} \left(\frac{\tau}{4}\|(\widetilde{\fb}_j)_{d+1}-(\fb_j)_{d+1}\|_{\infty} + 2\|(\widetilde{\fb}_j)_{1:d}-(\fb_j)_{1:d}\|_2\right) \nonumber\\
		\leq & \sup_{\xb\in \cM(q)}\sum_{j=1}^{C_{\cM}} \left(\frac{\tau}{4}\|\widetilde{\rho}_j(\xb)-\rho_j(\xb)\|_{\infty} + 2\|\widetilde{\phi}_j(\xb)-\phi_j\circ\pi(\xb)\|_2\right) \nonumber\\
		\leq & \sum_{j=1}^{C_{\cM}} \left(\frac{\tau}{4}\varepsilon_1+2\sqrt{d}\varepsilon_1\right) \nonumber\\
		\leq & C_{\cM}\left(\frac{\tau}{4}+2\sqrt{d}\right)\varepsilon_1.
		\label{eq.multi.G.err.term2}
	\end{align}
	In the above, we used $\phi_j^{-1}$ is Lipschitz continuous in the third inequality, $\phi_j^{-1}$ is bounded by $\frac{\tau}{4}$  and $|\rho_j|\leq 1$ in the fourth inequality, the definition of $\widetilde{\fb}_j$, $\fb_j$ ((\ref{eq.ftilde}) and (\ref{eq.ecoder.multi}), respectively) in the fifth inequality, and (\ref{eq.multi.phi}) and (\ref{eq.multi.rho}) in the sixth inequality.
	
	Substituting (\ref{eq.multi.G.err.term2}) into (\ref{eq.multi.G.err.1}) and setting $\varepsilon_1=\frac{\varepsilon}{2C_{\cM}\left(\frac{\tau}{4}+2\sqrt{d}\right)}$ and $\varepsilon_3=\varepsilon/2$ give rise to 
	\begin{align*}
		\sup_{\xb\in\cM(q)}\|\widetilde{\scrG}(\xb)-G(\xb)\|_{\infty}\leq \frac{\varepsilon}{2}+  \frac{\varepsilon}{2}=\varepsilon.
	\end{align*}
	
	\noindent $\bullet$ {\bf Network architectures.}
	For $\widetilde{\scrE}$, we have $\widetilde{\scrE}\in \cF^{\scrE}(D,C_{\cM}(d+1);L_{\scrE}, p_{\scrE}, K_{\scrE}, \kappa_{\scrE}, R_{\scrE})$ with
	\begin{align*}
		&L_{\scrE}=O(\log^2 \varepsilon^{-1}+\log D), \ p_{\scrE}=O(D\varepsilon^{-d}), \ K_{\scrE}=O((D\log D)\varepsilon^{-d}\log^2 \varepsilon^{-1}), \nonumber\\
		& \kappa_{\scrE}=O(\varepsilon^{-2}), \ R_{\scrE}=\tau/4.
	\end{align*}
	The constant hidden in $O$ depends on $d,\tau,q,B,C_{\cM}$ and the volume of $\cM$.
	
	For $\widetilde{\scrD}$, it consists of $\widetilde{\times}$ and $\widetilde{g}_j$:
	\begin{itemize}
		\item $\widetilde{\times}$: It has depth of $O(\log \varepsilon^{-1})$, width bounded by $6D$ and $O(D\log \varepsilon^{-1})$ parameters. All parameters are bounded by $B^2$.
		\item $\widetilde{g}_j$: It has depth $O(\log^2 \varepsilon^{-1}) +\log D$, width of $O(D\varepsilon^{-d})$ and $O(D\varepsilon^{-d} \log^2 \varepsilon +D\log D)$. All weight parameters are bounded by $\varepsilon^{-2}$.
	\end{itemize}
	Therefore, we have $\widetilde{\scrD}\in \cF^{\scrD}(C_{\cM}(d+1),D;L_{\scrD}, p_{\scrD}, K_{\scrD}, \kappa_{\scrD}, R_{\scrD})$ with
	\begin{align*}
		&L_{\scrD}=O(\log^2 \varepsilon^{-1}+\log D), \ p_{\scrD}=O(D\varepsilon^{-d}), \ K_{\scrD}=O(D\varepsilon^{-d} \log^2 \varepsilon +D\log D), \nonumber\\
		& \kappa_{\scrD}=O(\varepsilon^{-1}), \ R_{\scrD}=B.
	\end{align*}
	The constant hidden in $O$ depends on $d,\tau,q,B,C_{\cM}$ and the volume of $\cM$.
\end{proof}

\subsection{Proof of Lemma \ref{lem.Mcover}}
\label{proof.Mcover}
\begin{proof}[Proof of Lemma \ref{lem.Mcover}]
    Denote the packing number of $\cM$ using sets in the form of $\cM\cap B_r(\cbb)$ for $\cbb\in\cM$ by $P_{\cM}(r)$. By \citet[Lemma 5.3]{niyogi2008finding}, we have
    \begin{align}
        P_{\cM}(r)\leq \frac{|\cM|}{\cos^d(\arcsin \frac{r}{2\tau}) |B_r^d|}.
    \end{align}
    According to \citet[Lemma 5.2]{niyogi2008finding}, we get
    \begin{align}
        C_{\cM}\leq P_{\cM}(r/2)\leq \frac{|\cM|}{\cos^d(\arcsin \frac{r}{4\tau}) |B_r^d|}=O(r^{-d}).
    \end{align}
\end{proof}

\subsection{Proof of Lemma \ref{lem.rho}}
\label{proof.rho}
\begin{proof}[Proof of Lemma \ref{lem.rho}]
	Define $\eta(\xb)=\begin{bmatrix}
		\eta_1(\xb) & \cdots & \eta_{c}(\xb)
	\end{bmatrix}^{\top}$. By \citet[Lemma 15]{cloninger2021deep}, for any $0<\varepsilon<1$, there exists a network architecture $\cF(D,c;L_1,p_1,K_1,\kappa_1,1)$ giving rise to a network $\widetilde{\eta}$ so that
	\begin{align*}
		\sup_{\xb\in \cM(q)} \|\widetilde{\eta}(\xb)-\eta(\xb)\|_{1}\leq \varepsilon.
	\end{align*}
	Such a network has
	\begin{align*}
		L_1=O(\log^2 \varepsilon^{-1} + \log D), \ p_1=O(D\varepsilon^{-1}), \ K_1=O((D\log D)\varepsilon^{-1}\log^2 \varepsilon^{-1}), \kappa_1=O(\varepsilon^{-2}).
	\end{align*}
	The constant hidden in $O$ depends on $c,d,\tau,q,B,C_{\cM}$ and the volume  of $\cM$.
	
	Note that $\rho_j$ is the sum of the elements in the output of $\eta$ that have index in $\cI_j$. We construct $\widetilde{\rho}_j$ by appending $\widetilde{\eta}$ by a layer summing up the corresponding elements. Let $\wb\in \RR^c$ so that $\wb_k=1$ if $k\in \cI_j$, and $\wb_k=0$ otherwise. The definition of $\wb$ is an indicator for the $k$'s in $I_j$. For different $j$'s, we have different $I_j$'s and thus different $\wb$'s. We construct $\widetilde{\rho}_j$ as
	\begin{align*}
		\widetilde{\rho}_j(\xb)=\wb\cdot \ReLU(\widetilde{\eta}(\xb)).
	\end{align*}
	We have $\widetilde{\rho}_j\in \cF(1,L,p,K,\kappa,1)$ for $\cF(1,L,p,K,\kappa,1)$ defined in Lemma \ref{lem.rho} and
	\begin{align*}
		\sup_{\xb\in \cM(q)}|\widetilde{\rho}_j-\rho_j| =& \sup_{\xb\in \cM(q)} \sum_{k\in \cI_j} |\widetilde{\eta}_k-\eta_k| \nonumber\\
		\leq &\sup_{\xb\in \cM(q)} \sum_{k=1}^c |\widetilde{\eta}_k-\eta_k| \nonumber\\
		\leq & \varepsilon.
	\end{align*}

\end{proof}

\end{document}